\documentclass[graybox]{svmult} 
\usepackage[british]{babel}
\usepackage[T1]{fontenc}
\usepackage{amssymb,amsmath,bm}
\usepackage{mathtools}
\usepackage{nicefrac}
  \mathtoolsset{mathic}
  \allowdisplaybreaks
\usepackage{graphicx}
\usepackage{makeidx}
\usepackage{multicol}
\usepackage[bottom]{footmisc}
\usepackage{newtxtext,newtxmath}
\usepackage{hyperref,url}
\usepackage[all]{hypcap} 
\usepackage{xcolor}
\usepackage{enumitem}

\DeclarePairedDelimiter{\group}{(}{)}

\DeclarePairedDelimiter{\set}{\{}{\}}
\DeclarePairedDelimiter{\abs}{\vert}{\vert}

\newcommand{\naturals}{\mathbb{N}}

\newcommand{\reals}{\mathbb{R}}
\newcommand{\posreals}{\reals_{>0}}

\newcommand{\gbls}[1][]{\mathscr{L}_{#1}}
\newcommand{\sgbls}[1][]{\mathscr{S}_{#1}}
\newcommand{\gbl}{f}
\newcommand{\altgbl}{g}
\newcommand{\altgbltoo}{h}
\newcommand{\gbli}[1][]{f_{#1}}
\newcommand{\altgbli}[1][]{g_{#1}}

\newcommand{\optset}[1][]{A_{#1}}
\newcommand{\altoptset}[1][]{B_{#1}}
\newcommand{\altoptsettoo}[1][]{C_{#1}}
\newcommand{\optsets}{\mathscr{Q}}

\newcommand{\rejectset}[1][]{K_{#1}}

\newcommand{\rejectsets}{\mathbf{K}}

\newcommand{\choicefun}[1][]{C_{#1}}
\newcommand{\rejectfun}[1][]{R_{#1}}

\newcommand{\cset}[3][]{\set[#1]{#2\colon#3}}
\newcommand{\ind}[1]{\mathbb{I}_{#1}}

\newcommand{\ifandonlyif}{\Leftrightarrow}

\newcommand{\co}[1]{#1^\mathrm{c}}
\newcommand{\lowprev}[1][]{\underline{P}_{#1}}

\newcommand{\uppprev}[1][]{\overline{P}_{#1}}
\newcommand{\lowuppprev}[1][]{\overline{\underline{P}}_{#1}}
\newcommand{\linprev}[1][]{P_{#1}}
\newcommand{\altlinprev}[1][]{Q_{#1}}
\newcommand{\cohlowprevs}{\underline{\mathbf{P}}}
\newcommand{\linprevs}{\mathbf{P}}
\newcommand{\pspace}{\Omega}
\newcommand{\xvalues}{\mathscr{X}}
\newcommand{\yvalues}{\mathscr{Y}}
\newcommand{\xyvalues}{\xvalues\times\yvalues}
\newcommand{\zvalues}{\mathscr{Z}}
\newcommand{\partition}[1][]{\mathcal{P}_{#1}}
\newcommand{\ultrafilter}[2][\lowprev]{\mathscr{U}_{{#1},{#2}}}

\newcommand{\aevt}[1][]{E_{{#1}}}
\newcommand{\bevt}[1][]{F_{{#1}}}
\newcommand{\cevt}[1][]{G_{{#1}}}
\newcommand{\taevt}[1][]{\tilde{E}_{{#1}}}
\newcommand{\tbevt}[1][]{\tilde{F}_{{#1}}}
\newcommand{\xvar}[1][]{X_{{#1}}}
\newcommand{\yvar}[1][]{Y_{{#1}}}
\newcommand{\zvar}[1][]{Z_{{#1}}}
\newcommand{\xyvar}{{(\xvar,\yvar)}}
\newcommand{\lowdis}[1][\zvar]{\underline{P}_{#1}}

\newcommand{\uppdis}[1][\zvar]{\overline{P}_{#1}}

\newcommand{\dis}[1][\zvar]{P_{#1}}

\DeclareMathOperator{\posi}{posi}

\makeindex

\begin{document}

\title*{On a notion of independence proposed by Teddy Seidenfeld}
\author{Jasper De Bock and Gert de Cooman}
\authorrunning{De Bock and De Cooman}
\institute{Jasper De Bock \at Foundations Lab for Imprecise Probabilities, Ghent University, Technologiepark-Zwijnaarde 125, 9052 Belgium \email{jasper.debock@ugent.be}
\and Gert de Cooman \at Foundations Lab for Imprecise Probabilities, Ghent University, Technologiepark-Zwijnaarde 125, 9052 Belgium \email{gert.decooman@ugent.be}}
\maketitle

\abstract*{%
Teddy Seidenfeld has been arguing for quite a long time that binary preference models are not powerful enough to deal with a number of crucial aspects of imprecision and indeterminacy in uncertain inference and decision making.
It is at his insistence that we initiated our study of so-called sets of desirable option sets, which we have argued elsewhere provides an elegant and powerful approach to dealing with general, binary as well as non-binary, decision-making under uncertainty.
We use this approach here to explore an interesting notion of irrelevance (and independence), first suggested by Seidenfeld in an example intended as a criticism of a number of specific decision methodologies based on (convex) binary preferences. 
We show that the consequences of making such an irrelevance or independence assessment are very strong, and might be used to argue for the use of so-called mixing choice functions, and E-admissibility as the resulting decision scheme.}

\abstract{%
Teddy Seidenfeld has been arguing for quite a long time that binary preference models are not powerful enough to deal with a number of crucial aspects of imprecision and indeterminacy in uncertain inference and decision making.
It is at his insistence that we initiated our study of so-called sets of desirable option sets, which we have argued elsewhere provides an elegant and powerful approach to dealing with general, binary as well as non-binary, decision-making under uncertainty.
We use this approach here to explore an interesting notion of irrelevance (and independence), first suggested by Seidenfeld in an example intended as a criticism of a number of specific decision methodologies based on (convex) binary preferences. 
We show that the consequences of making such an irrelevance or independence assessment are very strong, and might be used to argue for the use of so-called mixing choice functions, and E-admissibility as the resulting decision scheme.}


\section{Context and introduction}\label{sec:introduction}
In much of our earlier work on the foundations of imprecise---or indeterminate \cite{levi1999:isipta:imprecise:indeterminate}---probabilities \cite{augustin2013:itip,walley1991} we availed ourselves of \emph{binary} preference orders between uncertain rewards to model a subject's decisions under uncertainty; see \cite{debock2015:thesis,debock2015:credal:nets,decooman2015:coherent:predictive:inference,cooman2011b,cooman2010,quaeghebeur2015:statement} for a few representative examples.
In the field, the monikers `desirability' and `sets of desirable gambles' are typically used to describe uncertainty models involving such (strict) binary preference orders \cite{couso2011,decooman2015:coherent:predictive:inference,cooman2010,quaeghebeur2012:itip,walley2000}. 
In earlier work, Seidenfeld et al.~\cite{seidenfeld1995} also introduced the term `favourability' for this. 

Since the publication of that work, Teddy Seidenfeld has been developing arguments in favour of a more involved approach to uncertainty modelling and decision making. 
If we really want to take indecision and imprecision seriously, he has insisted, we need to abandon binary preference models in favour of more general choice functions, as for instance described in \cite{seidenfeld2010}.
For this reason, one of us (Gert) started to work with Arthur Van Camp---his then PhD student---on exploring the connections between choice function theory and desirability. 
This led to a number of joint papers \cite{vancamp2018:exchangeability,2018vancamp:lexicographic,vancamp2018:indifference}, and, eventually, to Arthur's PhD Thesis \cite{2017vancamp:phdthesis}.

Inspired by that work, the two of us decided to explore this connection further. 
A key insight we had, is that choice functions, when interpreted appropriately, can express statements such as ``\emph{at least one} of these preferences is true''. 
Since a desirable gamble is by definition a preference assessment---it is an uncertain reward that is strictly preferred to the status quo---this suggests that choice functions can deal with `OR'-statements between assessments of desirability.
In contrast, the language of sets of desirable gambles typically only deals with `AND'-statements between such assessments.\footnote{How to also deal with `NOT' in this and related languages, was studied in quite some detail by one of us in an earlier collaboration \cite{quaeghebeur2015:statement}.}
This observation led us to the intriguing idea that general choice functions might be interpreted, axiomatised and represented using the language of desirability, and that at the same time, they could enrich this language with `OR'-statements.
Investigating this idea and confirming our suspicions in all the necessary detail has been part of an ongoing project, with a number of papers nearing completion.
Early versions, which the present discussion is based on, have been published in conference proceedings \cite{ipmu2020debock,debock2018,debock2019:interpretation,ipmu2020decooman}, and more detailed versions with proofs are also available on ArXiv \cite{ipmu2020debock:arxiv,debock2018:arXiv,debock2019:interpretation:arxiv,ipmu2020decooman:arxiv}.
We will summarise the relevant ideas and results in Section~\ref{sec:choice-functions} further on.

Our results so far have led us to agree with Seidenfeld's criticism, and to follow him in moving from binary preference models to choice functions. 
But they have not led us to abandon desirability.
On the contrary: on our account, desirability is also very well suited for describing and interpreting non-binary choice. 
This interpretation---that (not) choosing an option from an option set with more than two elements can be brought back to an `OR' of desirability statements---inspires a set of axioms which allows us to cover much---if not all---of the literature on the subject that we have come across.

Simply moving towards general choice functions doesn't immunise us against all aspects of Seidenfelds's criticism, however. 
For it is not merely the systematic use of binary decision schemes that he has been arguing against, it is also---and perhaps foremost---some of their features, which may also be shared by some types of non-binary decision schemes. 
In fact, he has a treasure trove of intricate little examples that he likes to pick apart other people's pet theories about rational decision making with.
We have no doubt that some of them may still be brought to bear on specific decision schemes within our desirability-based theory of choice functions.
In one of his examples, which he typically mobilises to cast doubt on the indiscriminate assumption of convexity for a set of indeterminate probabilities, he introduced {\itshape en passant} a requirement for `independence' that hasn't stopped fascinating us since the fateful day he sent us a few hand-outs explaining the basic ideas behind it. 
Whereas the hand-outs are not publicly available, the main idea expressed in them is, because he and his colleagues have also published a similar example in \cite[Section 4]{seidenfeld2010}. 
It involves the following very intuitive `rationality requirement' about the value of `independent information', namely, that it ought to have none:
\begin{quote}
When two events, \(\aevt\) and \(\bevt\), are `independent' then it is not reasonable to \emph{spend resources} in order to use the state of one, \(\aevt\) versus \(\co{\aevt}\), to decide between two gambles that depend solely on the other event, \(\bevt\) versus \(\co{\bevt}\).
\end{quote}
Rather than reproduce his specific example here in its full detail, we will rephrase his requirement in a more general abstract form, and without the symmetry that is implicit in his formulation. 
We will refer to this asymmetric version as (an assessment of) \emph{S-irrelevance}. 
The main goal of this paper will be to study its implications. 

We consider a possibility space~\(\pspace\), and two \emph{events}~\(\aevt,\bevt\subseteq\pspace\).\footnote{We will use the language of events, rather than propositions, to express the things we are uncertain about, but the difference is immaterial for what we have in mind.}
The event~\(\bevt\) could for instance refer to a(n unknown) medical condition of a patient in a Brussels hospital, and the event~\(\aevt\) could refer to (unknown) specific weather conditions at the South Pole.
\emph{Gambles} are uncertain rewards expressed in units of some predetermined linear utility,\footnote{Our results can be developed using horse lotteries, but we opt here for a simplified version.} modelled as bounded maps \(\gbl\colon\pspace\to\reals\).
We denote the set of all such maps by~\(\gbls(\pspace)\), or more simply by~\(\gbls\) when it is clear from the context what the domain of the gambles is.

A \emph{gamble on the occurrence of~\(\bevt\)} is a gamble of the type
\begin{equation*}
\lambda\ind{\bevt}+\mu\ind{\co{\bevt}}
=\begin{cases}
\lambda&\text{ if \(\bevt\) occurs}\\
\mu&\text{ if \(\bevt\) doesn't occur},
\end{cases}
\end{equation*}
where \(\lambda\) and \(\mu\) are real numbers, \(\co{\bevt}\coloneqq\pspace\setminus\bevt\) is the complement of~\(\bevt\), and \(\ind{\bevt}\) is the indicator of~\(\bevt\): the gamble that assumes the value~\(1\) on~\(\bevt\) and \(0\) elsewhere.
We denote by
\begin{equation*}
\gbls[\bevt]\coloneqq\cset{\lambda\ind{\bevt}+\mu\ind{\co{\bevt}}}{\lambda,\mu\in\reals}
\end{equation*}
the set of all such gambles.

To formalise his idea, Seidenfeld considers two gambles~\(\gbl,\altgbl\in\gbls[\bevt]\) on the occurrence of~\(\bevt\). 
In our example, they could for instance represent the uncertain rewards for two possible courses of treatment for our patient in the Brussels hospital, whose actual rewards are determined by her actual (but unknown) medical condition. 

We can use these two gambles to construct a composite gamble \(\ind{\aevt}\gbl+\ind{\co{\aevt}}\altgbl\), whose outcome also depends on the state of the event \(\aevt\): it yields the uncertain reward \(\gbl\) if \(\aevt\) occurs, and \(\altgbl\) if \(\aevt\) doesn't occur.
In our example, this would correspond to taking either the first or the second treatment, depending on the weather condition at the South Pole, and \(\ind{\aevt}\gbl+\ind{\co{\aevt}}\altgbl\) is then the reward function for this composite treatment.

We now consider the problem of choosing between the gambles in the collection
\begin{equation*}
\optset[\epsilon]\coloneqq\set{\gbl,\altgbl,\ind{\aevt}\gbl+\ind{\co{\aevt}}\altgbl-\epsilon}
\text{ for some real constant \(\epsilon>0\).}
\end{equation*}
The third gamble consists in paying a price \(\epsilon>0\) in order to find out the status of~\(\aevt\) which then determines whether we get the uncertain reward~\(\gbl\) or \(\altgbl\) on the outcome of~\(\bevt\).
Seidenfeld's requirement states that when a subject believes \(\aevt\) and~\(\bevt\) to be `independent', then the third gamble must never be chosen, for any \(\epsilon>0\).
Or alternatively, in a language that stresses rejection rather than choice: the third option must be rejected from the set of options \(\optset[\epsilon]\) for all \(\epsilon>0\).\footnote{In both his hand-outs and the above-mentioned paper \cite{seidenfeld2010}, Seidenfeld argues, similarly, for what he calls the \emph{inadmissibility} of the third option.}

If, as we will explain in Section~\ref{sec:choice-functions}, we consider choice or rejection statements to provide information about strict preferences between gambles, then this requirement states that at least one option in the set \(\set{\gbl,\altgbl}\) must be preferred over the rejected option~\(\ind{\aevt}\gbl+\ind{\co{\aevt}}\altgbl-\epsilon\), which amounts to OR-ing these two preference assessments:
\begin{equation*}
\gbl\text{ is preferred over }\ind{\aevt}\gbl+\ind{\co{\aevt}}\altgbl-\epsilon
\quad\text{OR}\quad
\altgbl\text{ is preferred over }\ind{\aevt}\gbl+\ind{\co{\aevt}}\altgbl-\epsilon.
\end{equation*}
In this paper, we intend to explore the consequences of making such assessments, using our new approach to coherent choice functions, which is, as we have already mentioned, eminently suited for dealing with such `OR'-statements.

Thus, in developing this independence idea, we can make two lines of research come together, both of which were inspired by Teddy Seidenfeld.

We outline our approach to coherent choice in Section~\ref{sec:choice-functions}, and derive the far-reaching consequences of imposing the above-mentioned `independence' requirement in Section~\ref{sec:S-irrelevance:for:events}.
We extend the discussion from events to variables in Section~\ref{sec:S-irrelevance:for:variables}, and dwell on the implications of our findings in Section~\ref{sec:discussion}.

\section{A crash course in desirability-based choice functions}\label{sec:choice-functions}
A \emph{choice function}~\(\choicefun\) is a set-valued operator on sets of options. 
In particular, for any set of options \(\optset\), the corresponding value of~\(\choicefun\) is a subset \(\choicefun\group{\optset}\) of~\(\optset\). 
We will be considering throughout the special case where these options are gambles in~\(\gbls(\pspace)\): bounded real-valued maps on~\(\pspace\), interpreted as uncertain rewards.\footnote{For a more general approach to desirability-based choice functions, where options can take values in an abstract vector space, we refer the interested reader to one of our earlier papers~\cite{debock2019:interpretation,debock2019:interpretation:arxiv}.} 
The option sets~\(\optset\) are furthermore taken to be finite, and we denote the set of all such finite subsets of~\(\gbls(\pspace)\) by~\(\optsets(\pspace)\).
Again, if it is clear from the context what the possibility space is, we will use the simpler notations~\(\gbls\) and~\(\optsets\). 

Gambles can be ordered by the point-wise ordering, where \(\gbl\geq\altgbl\) means that \(\gbl(\omega)\geq\altgbl(\omega)\) for all \(\omega\in\pspace\), and \(\gbl>\altgbl\) means that \(\gbl\geq\altgbl\) but \(\gbl\neq\altgbl\).
We will also use the notation \(\gbl\succ\altgbl\) to mean that \(\inf(\gbl-\altgbl)>0\).

The purpose of a choice function then, is to represent a subject's choice between the options in~\(\optset\), for any \(\optset\in\optsets\). 
The terminology can be a bit misleading, though, because to say that \(\choicefun(\optset)=\altoptset\) is not taken to mean that all options in~\(\altoptset\) are chosen. 
Rather, it means that, based on the available information, our subject is only disposed to rule out the options in~\(\optset\setminus\altoptset\), but remains undecided about the remaining options in~\(\altoptset\). 
For this reason, it makes sense to focus on the options that are rejected---as in `not chosen'---and to consider the corresponding \emph{rejection function}~\(\rejectfun\), defined by~\(\rejectfun(\optset)\coloneqq\optset\setminus\choicefun(\optset)\) for all \(\optset\in\optsets\).

Our interpretation for rejection---and hence also choice---functions now goes as follows. 
Consider a subject whose uncertainty is represented by a rejection function~\(\rejectfun\), or equivalently, by a choice function~\(\choicefun\). 
Then for a given option set \(\optset\in\optsets\), the statement that an option~\(\gbl\in\optset\) is rejected from~\(\optset\)---that \(\gbl\in\rejectfun\group{\optset}\) or \(\gbl\notin\choicefun(\optset)\)---is taken to mean that \emph{there is at least one option~\(\altgbl\) in~\(\optset\) that our subject strictly prefers over \(\gbl\)}. 

The connection with the language of desirability is now almost immediate, because it is eminently suited for dealing with binary preference statements such as ``the gamble~\(\altgbl\) is strictly preferred over the gamble~\(\gbl\)''. 
In terms of desirability, this simply means that \(\altgbl-\gbl\) is desirable~\cite{quaeghebeur2012:itip,walley2000}. 
By applying this to our interpretation for rejection in terms of preferences, we obtain an equivalent interpretation in terms of desirability: \(\gbl\) is rejected from~\(\optset\) if the option set
\begin{equation*}
\optset\ominus\gbl\coloneqq\cset{\altgbl-\gbl}{\altgbl\in\optset\setminus\set{\gbl}}
\end{equation*}
contains at least one desirable gamble. 
So we find that under our interpretation, the study of choice and rejection functions reduces to the study of sets of gambles that contain at least one desirable gamble.

In order to formalise this, we have introduced the concept of a \emph{desirable option set}: a set~\(\optset\in\optsets\) that contains at least one desirable gamble. 
A subject's uncertainty can then be described by means of a set of such desirable option sets: a set~\(\rejectset\subseteq\optsets\) of option sets~\(\optset\) that she assesses to be desirable, in the sense that according to her beliefs, every \(\optset\in\rejectset\) contains at least one desirable gamble. 
For any such set of desirable option sets~\(\rejectset\), the corresponding rejection function and choice function are then defined by
\begin{equation}\label{eq:choiceandrejectionfromK}
\rejectfun(\optset)\coloneqq\cset{\gbl\in\optset}{\optset\ominus\gbl\in\rejectset}
\text{ and }
\choicefun(\optset)\coloneqq\cset{\gbl\in\optset}{\optset\ominus\gbl\notin\rejectset}
\end{equation}
In the remainder of this contribution, we will focus mainly on sets of desirable options sets~\(\rejectset\), and will consider choice and rejection functions as derived objects, obtained by Equation~\eqref{eq:choiceandrejectionfromK}. 
In particular, we will focus on sets of desirable option sets that are \emph{coherent}, in the sense that they satisfy the following rationality criteria for the beliefs---or behavioural dispositions---expressed by~\(\rejectset\). 
We will use~`\((\lambda,\mu)>0\)' as a shorthand notation for `\(\lambda\geq0\), \(\mu\geq0\) and \(\lambda+\mu>0\)'.

\begin{definition}[Coherence for sets of desirable option sets]\label{def:coherence:rejectset}
A set of desirable option sets~\(\rejectset\subseteq\optsets\) is called \emph{coherent} if it satisfies the following axioms:
\begin{enumerate}[label=\(\mathrm{K}_{\arabic*}\).,ref=\(\mathrm{K}_{\arabic*}\),leftmargin=*,start=0]
\item\label{ax:rejects:removezero} if \(\optset\in\rejectset\) then also \(\optset\setminus\set{0}\in\rejectset\), for all \(\optset\in\optsets\);
\item\label{ax:rejects:nozero} \(\set{0}\notin\rejectset\);
\item\label{ax:rejects:pos} \(\set{\gbl}\in\rejectset\), for all \(\gbl\in\gbls\) with \(\inf\gbl>0\);
\item\label{ax:rejects:cone} if \(\optset[1],\optset[2]\in\rejectset\) and if, for all \(\gbl\in\optset[1]\) and~\(\altgbl\in\optset[2]\), \((\lambda_{\gbl,\altgbl},\mu_{\gbl,\altgbl})>0\), then also
\begin{equation*}
\cset{\lambda_{\gbl,\altgbl}\gbl+\mu_{\gbl,\altgbl}\altgbl}{\gbl\in\optset[1],\altgbl\in\optset[2]}
\in\rejectset;
\end{equation*}
\item\label{ax:rejects:mono} if \(\optset[1]\in\rejectset\) and~\(\optset[1]\subseteq\optset[2]\), then also \(\optset[2]\in\rejectset\), for all \(\optset[1],\optset[2]\in\optsets\).
\end{enumerate}
\end{definition}
\noindent
This axiomatisation is entirely based on---and motivated by---our interpretation and the following three rationality principles for a notion of desirability:
\begin{enumerate}[label=\(\mathrm{d}_{\arabic*}\).,ref=\(\mathrm{d}_{\arabic*}\),start=1,leftmargin=*]
\item\label{ax:desirability:nozero} 
\(0\) is not desirable;
\item\label{ax:desirability:pos} if \(\gbl\succ0\), or in other words, \(\inf\gbl>0\), then \(\gbl\) is desirable;\footnote{Our general approach \cite{debock2019:interpretation,debock2019:interpretation:arxiv} allows for more general `background orderings' \(\succ\) to replace the ordering based on `\(\inf\gbl>0\)' considered here.}
\item\label{ax:desirability:cone} if \(\gbl,\altgbl\) are desirable and~\((\lambda,\mu)>0\), then \(\lambda\gbl+\mu\altgbl\) is desirable.
\end{enumerate}
For a motivation and discussion of these principles, we refer to~\cite{quaeghebeur2012:itip,walley2000}.
For a detailed explanation of why \ref{ax:desirability:nozero}--\ref{ax:desirability:cone} indeed naturally lead to \ref{ax:rejects:removezero}--\ref{ax:rejects:mono}, we refer to~\cite{debock2019:interpretation}, which also contains a small example that illustrates the use of our workhorse axiom~\ref{ax:rejects:cone}. 
More generally, that same reference~\cite{debock2019:interpretation} also provides much more information about---and motivation for---the framework that we summarise here. 

For our present purposes, it will suffice to focus on a number of special cases that play a central role in this paper: choice functions based on linear, and on lower, previsions.

\subsection{Choice functions based on linear previsions}\label{sec:choicefromlinear}
Perhaps the best-known method for choosing between uncertain rewards, is to choose those options that have the highest expected utility with respect to some given probability measure. 
This measure is often taken to be countably additive, but we will not impose this restriction here and work with finitely additive probability measures instead, defined on all events \(\aevt\subseteq\pspace\). 
The expectation operators that correspond to such measures are linear previsions.

\begin{definition}[Linear prevision]\label{def:linearprevision}
A \emph{linear prevision}~\(\linprev\) on~\(\gbls\) is a real-valued map on~\(\gbls\) that satisfies
\begin{enumerate}[label=\(\mathrm{P}_{\arabic*}\).,ref=\(\mathrm{P}_{\arabic*}\),leftmargin=*,start=1]
\item\label{ax:linprev:inf} \(\linprev(\gbl)\geq\inf\gbl\) for all \(\gbl\in\gbls\);
\item\label{ax:linprev:homo} \(\linprev(\lambda\gbl)=\lambda\linprev(\gbl)\) for all \(\lambda\in\reals\) and~\(\gbl\in\gbls\);
\item\label{ax:linprev:additive} \(\linprev(\gbl+\altgbl)=\linprev(\gbl)+\linprev(\altgbl)\) for all \(\gbl,\altgbl\in\gbls\).
\end{enumerate}
We denote the set of all linear previsions on~\(\gbls\) by~\(\linprevs\). 
\end{definition}
Conversely, every linear prevision~\(\linprev\) has a corresponding finitely additive probability measure, also denoted by~\(\linprev\), and defined by~\(\linprev(\aevt)\coloneqq\linprev(\ind{\aevt})\) for all \(\aevt\subseteq\pspace\).

For any given linear prevision---or, equivalently, any finitely additive probability measure---\(\linprev\), we now let \(\choicefun[\linprev]\) be the choice function that corresponds to maximising expected utility. 
For all \(\optset\in\optsets\), it is defined by\footnote{It is also customary in much of the literature to furthermore remove from a choice set \(\choicefun(\optset)\) those options that are dominated by other options in \(\optset\) for the point-wise ordering \(\geq\) of options. We will leave this implicit here, as an operation that can always be performed \emph{afterwards}.}
\begin{equation}\label{eq:choicefromlinear}
\choicefun[\linprev](\optset)
\coloneqq\cset{\gbl\in\optset}{(\forall\altgbl\in\optset)\linprev(\gbl)\geq\linprev(\altgbl)}.
\end{equation}
That this is a special case of our more general framework can be seen by defining the set of desirable option sets
\begin{equation}\label{eq:Kfromlinprev}
\rejectset[\linprev]\coloneqq\cset{\optset\in\optsets}{(\exists\gbl\in\optset)\linprev(\gbl)>0},
\end{equation}
which is easily verified to be coherent. 
By applying Equation~\eqref{eq:choiceandrejectionfromK}, we find that the corresponding choice function is indeed given by~\(\choicefun[\linprev]\): for any \(\optset\in\optsets\) and~\(\gbl\in\optset\), we see that
\begin{align*}
\optset\ominus\gbl\notin\rejectset[\linprev]
\ifandonlyif\neg(\exists\altgbl\in\optset\setminus\set{\gbl})\linprev(\altgbl-\gbl)>0
&\ifandonlyif(\forall\altgbl\in\optset\setminus\set{\gbl})\linprev(\altgbl-\gbl)\leq0\\
&\ifandonlyif(\forall\altgbl\in\optset\setminus\set{\gbl})\linprev(\gbl)\geq\linprev(\altgbl)\\
&\ifandonlyif(\forall\altgbl\in\optset)\linprev(\gbl)\geq\linprev(\altgbl).
\end{align*}
It is clear that the choice models~\(\rejectset[\linprev]\) and~\(\choicefun[\linprev]\) are \emph{binary}, in the sense that they are completely determined by a binary strict preference relation on gambles \(\gbl,\altgbl\in\gbls\), in this case expressed by \(\linprev(\gbl)>\linprev(\altgbl)\).

More generally, we can replace the linear prevision, or probability measure, \(\linprev\) by a set \(\mathscr{P}\subseteq\linprevs\) of such previsions. 
In that case, one possible approach to decision making is to apply Levi's E-admissibility criterion \cite{levi1980a,troffaes2007}, which amounts to considering the union of the choices of the individual \(\linprev\in\mathscr{P}\), or equivalently, rejecting the options that are rejected under every \(\linprev\in\mathscr{P}\), typically leading to a non-binary choice model.\footnote{Strictly speaking, Levi only introduced, and argued for, this criterion in a context where he required the set \(\mathscr{P}\) to be \emph{convex}. We will still use the term `E-admissibility' even when \(\mathscr{P}\) is not convex. As mentioned before, we also leave the removal of \(\geq\)-dominated options implicit, as something that can be done afterwards.}
We let \(\choicefun[\mathscr{P}]\) be the choice function that corresponds to this decision criterion, defined by
\begin{align}
\choicefun[\mathscr{P}](\optset)
\coloneqq&\bigcup\cset{\choicefun[\linprev](\optset)}{\linprev\in\mathscr{P}}\notag\\
=&\cset[\big]{\gbl\in\optset}{(\exists\linprev\in\mathscr{P})(\forall\altgbl\in\optset)\linprev(\gbl)\geq\linprev(\altgbl)}
\text{ for all \(\optset\in\optsets\).}
\label{eq:choicefromEadmissibility}
\end{align}
This too corresponds to a special case of our framework, as can be seen by applying Equation~\eqref{eq:choiceandrejectionfromK} to the set of desirable option sets
\begin{equation}\label{eq:KfromEadmissibility}
\rejectset[\mathscr{P}]\coloneqq\bigcap_{\linprev\in\mathscr{P}}\rejectset[\linprev]
=\cset{\optset\in\optsets}{(\forall\linprev\in\mathscr{P})(\exists\gbl\in\optset)\linprev(\gbl)>0}.
\end{equation}
That \(\rejectset[\mathscr{P}]\) is coherent can be seen by observing that coherence is preserved under taking (non-empty) intersections of sets of desirable option sets.

\subsection{Choice functions based on lower previsions}\label{sec:choicefromlower}
While intuitive and straightforward,  E-admissibility is not the only possible generalisation of expectation maximisation. 
Within the field of imprecise probabilities, (Walley--Sen) maximality~\cite{troffaes2007,walley1991} is another extension that is often adopted; this too, as we will see, corresponds to a special case of our framework.

The uncertainty models to which the decision criterion of maximality is typically applied are not linear previsions, but rather a generalisation called \emph{coherent lower previsions}. 
We will only give a very brief account of them here; many more details about their interpretation and mathematical properties can be found in \cite{troffaes2013:lp,walley1991}.

\begin{definition}[Coherent lower prevision]\label{def:lowerprevision}
A \emph{coherent lower prevision}~\(\lowprev\) on~\(\gbls\) is a real-valued map on~\(\gbls\) that satisfies
\begin{enumerate}[label=\(\mathrm{LP}_{\arabic*}\).,ref=\(\mathrm{LP}_{\arabic*}\),leftmargin=*,start=1]
\item\label{ax:lowprev:inf} \(\lowprev(\gbl)\geq\inf\gbl\) for all \(\gbl\in\gbls\);
\item\label{ax:lowprev:homo} \(\lowprev(\lambda\gbl)=\lambda\lowprev(\gbl)\) for all \(\lambda\in\posreals\) and \(\gbl\in\gbls\);
\item\label{ax:lowprev:superadditive} \(\lowprev(\gbl+\altgbl)\geq\lowprev(\gbl)+\lowprev(\altgbl)\) for all \(\gbl,\altgbl\in\gbls\).
\end{enumerate}
We denote the set of all coherent lower previsions on~\(\gbls\) by~\(\cohlowprevs\).
\end{definition}
\noindent For any event \(\aevt\subseteq\pspace\), we will also call \(\lowprev(\aevt)\coloneqq\lowprev(\ind{\aevt})\) the \emph{lower probability} of~\(\aevt\) and \(\smash{\uppprev(\aevt)\coloneqq\uppprev(\ind{\aevt})}\) its \emph{upper probability}.

By comparing Definitions~\ref{def:linearprevision} and~\ref{def:lowerprevision}, we see that linear previsions indeed correspond to a special case of coherent lower previsions. 
In particular, they are coherent lower previsions \(\lowprev\) that are \emph{precise}, in the sense that they coincide with their conjugate upper prevision~\(\smash{\uppprev}\), defined by
\begin{equation*}
\uppprev(\gbl)\coloneqq-\lowprev(-\gbl)
\text{ for all \(\gbl\in\gbls\).}
\end{equation*}

\begin{proposition}[{\protect\cite[Section~2.3.6]{walley1991}}]
Let \(\lowprev\) be a coherent lower prevision on~\(\gbls\). 
Then \(\lowprev\) is a linear prevision on~\(\gbls\) if and only if \(\smash{\lowprev(\gbl)=\uppprev(\gbl)}\) for all \(\gbl\in\gbls\).
\end{proposition}

Besides their defining properties~\ref{ax:lowprev:inf}--\ref{ax:lowprev:superadditive}, coherent lower previsions also satisfy various other properties that are often conveniently used in proofs. 
We only mention a small selection and refer to~\cite[Section~2.6.1]{walley1991} and~\cite{troffaes2013:lp} for more extensive lists, proofs and further discussion:
\begin{enumerate}[label=\(\mathrm{LP}_{\arabic*}\).,ref=\(\mathrm{LP}_{\arabic*}\),leftmargin=*,start=4]
\item\label{ax:lowprev:infandsup} \(\inf\gbl\leq\lowprev(\gbl)\leq\uppprev(\gbl)\leq\sup\gbl\) for all \(\gbl\in\gbls\);
\item\label{ax:lowprev:monotonicity} if \(\gbl\leq\altgbl\) then \(\lowprev(\gbl)\leq\lowprev(\altgbl)\) and \(\uppprev(\gbl)\leq\uppprev(\altgbl)\) for all \(\gbl,\altgbl\in\gbls\);
\item\label{ax:lowprev:constantadditive} \(\lowprev(\gbl+\mu)=\lowprev(\gbl)+\mu\) for all \(\gbl\in\gbls\) and \(\mu\in\reals\);
\item\label{ax:lowprev:uniformconvergence} \(\lowprev(\gbli[n])\to\lowprev(\gbl)\) for all \(\gbli[n],\gbl\in\gbls\) such that \(\sup\abs{\gbli[n]-\gbl}\to0\);
\item\label{ax:lowprev:mixed:additivity} \(\lowprev(\gbl)+\lowprev(\altgbl)\leq\lowprev(\gbl+\altgbl)\leq\lowprev(\gbl)+\uppprev(\altgbl)\leq\uppprev(\gbl+\altgbl)\leq\uppprev(\gbl)+\uppprev(\altgbl)\) for all \(\gbl,\altgbl\in\gbls\).
\end{enumerate}
Observe that we have identified the real number~\(\mu\) with the gamble that assumes this constant value.

Given a coherent lower prevision~\(\lowprev\) on~\(\gbls\), we let \(\choicefun[\lowprev]\) be the choice function that is obtained by applying the criterion of maximality. 
For all \(\optset\in\optsets\), it is defined by
\begin{equation}\label{eq:choicefromlower}
\choicefun[\lowprev](\optset)
\coloneqq\cset[\big]{\gbl\in\optset}{(\forall\altgbl\in\optset)\lowprev(\altgbl-\gbl)\leq0}.
\end{equation}
The idea here is that an option~\(\gbl\in\optset\) is rejected from an option set~\(\optset\) if there is some (other) option~\(\altgbl\in\optset\) such that \(\lowprev(\altgbl-\gbl)>0\).\footnote{Again, we leave the removal of dominated options for the point-wise ordering~\(\geq\) implicit, as something that can be done afterwards.}
As for linear previsions, this leads to a binary choice model, as it is completely determined by this binary strict preference relation on gambles.

Under the behavioural interpretation of lower previsions~\cite[Section~2.3.1]{walley1991}, where the lower prevision of a gamble is interpreted as the supremum price for buying that gamble, \(\lowprev(\altgbl-\gbl)>0\) means that our subject is willing to pay a strictly positive price for the gamble \(\altgbl-\gbl\), or equivalently, that she is willing to pay a strictly positive price to replace the uncertain reward \(\gbl\) by~\(\altgbl\).

Alternatively, the statement that \(\lowprev(\altgbl-\gbl)>0\) can also be interpreted in terms of linear previsions, by considering the set of all linear previsions that dominate \(\lowprev\):
\begin{equation*}
\linprevs(\lowprev)
\coloneqq
\cset{\linprev\in\linprevs}{\linprev(\gbl)\geq\lowprev(\gbl)\text{ for all \(\gbl\in\gbls\)}}.
\end{equation*}
This set is non-empty, convex and closed with respect to the topology of point-wise convergence of bounded linear real functionals, and it furthermore has \(\lowprev\) and \(\smash{\uppprev}\) as its lower and upper envelopes~\cite[Section~3.3.3]{walley1991}:
\begin{equation}\label{eq:lowerandupperenvelope}
\lowprev(\gbl)=\min\cset{\linprev(\gbl)}{\linprev\in\linprevs(\lowprev)}
\text{ and }
\uppprev(\gbl)=\max\cset{\linprev(\gbl)}{\linprev\in\linprevs(\lowprev)}.
\end{equation}
The statement that \(\lowprev(\altgbl-\gbl)\) is strictly positive can therefore also be interpreted as stating that \(\linprev(\altgbl)>\linprev(\gbl)\) for all \(\linprev\in\linprevs(\lowprev)\). 
This leads to the following alternative characterisation of~\(\choicefun[\lowprev]\):
\begin{equation*}
 \choicefun[\lowprev](\optset)
 =\cset[\big]{\gbl\in\optset}{(\forall\altgbl\in\optset)(\exists\linprev\in\linprevs(\lowprev))\linprev(\gbl)\geq\linprev(\altgbl)}
\text{ for all \(\optset\in\optsets\).}
\end{equation*}
By comparing this expression with Equation~\eqref{eq:choicefromEadmissibility}, we see that, generally speaking, \(\choicefun[\lowprev]\) doesn't coincide with~\(\choicefun[\linprevs(\lowprev)]\), which illustrates that maximality and E-admissibility are distinct decision criteria. 
In the precise case, however, they do coincide and then reduce to expectation maximisation. 
For E-admissibility, this is immediate; for maximality, this can be seen by comparing Equations~\eqref{eq:choicefromlinear} and~\eqref{eq:choicefromlower} in the case that \(\lowprev=\linprev\) is linear.

That maximality is also a special case of our desirability-based theory of choice functions can be seen by defining the following set of desirable option sets
\begin{equation}\label{eq:Kfromlowprev}
\rejectset[\lowprev]\coloneqq\cset{\optset\in\optsets}{(\exists\gbl\in\optset)\lowprev(\gbl)>0}.
\end{equation}
It is easily verified to be coherent [this follows from~\ref{ax:lowprev:inf}--\ref{ax:lowprev:superadditive}], and its corresponding choice function is given by~\(\choicefun[\lowprev]\).

Since sets of desirable option sets of this kind are coherent, so are their intersections. 
With any non-empty set \(\mathscr{P}\subseteq\cohlowprevs\) of coherent lower previsions, we can therefore associate a coherent set of desirable option sets
\begin{equation}
\rejectset[\mathscr{P}]\coloneqq\bigcap\cset{\rejectset[\lowprev]}{\lowprev\in\mathscr{P}}
=\cset{\optset\in\optsets}{(\forall\lowprev\in\mathscr{P})(\exists\gbl\in\optset)\lowprev(\gbl)>0}
\label{eq:Kfromsetsoflower}
\end{equation}
and its corresponding choice function~\(\choicefun[\mathscr{P}]\), given by
\begin{align}
\choicefun[\mathscr{P}](\optset)
\coloneqq&\bigcup\cset{\choicefun[\lowprev](\optset)}{\lowprev\in\mathscr{P}}\notag\\
=&\cset[\big]{\gbl\in\optset}{(\exists\lowprev\in\mathscr{P})(\forall\altgbl\in\optset)\lowprev(\altgbl-\gbl)\leq0}
\text{ for all \(\optset\in\optsets\).}
\label{eq:choicefromsetsoflower}
\end{align}
If \(\mathscr{P}\) consists of linear previsions only, these respective expressions reduce to the Equations~\eqref{eq:KfromEadmissibility} and~\eqref{eq:choicefromEadmissibility} for E-admissibility. 
If \(\mathscr{P}\) consists of a single lower prevision, we obtain maximality, and if this single lower prevision is furthermore linear, we arrive at expectation maximisation. 
So we see that this class of choice functions contains all the special cases that we have seen so far. 
We will call all such choice models \emph{Archimedean}.

\begin{definition}\label{def:archimedeanity}
A set of desirable option sets \(\rejectset\) is called \emph{Archimedean} if there is some non-empty set \(\mathscr{P}\subseteq\cohlowprevs\) of coherent lower previsions such that \(\rejectset=\rejectset[\mathscr{P}]\). 
The largest such set~\(\mathscr{P}\) is then \(\cohlowprevs(\rejectset)\coloneqq\cset{\lowprev\in\cohlowprevs}{\rejectset\subseteq\rejectset[\lowprev]}\).
\end{definition}

\subsection{An axiomatic basis for working with linear and lower previsions}
In the remainder of this contribution, we will focus on Archimedean sets of desirable option sets and their corresponding choice functions, either in their full generality or in particular cases. 
However, rather than impose this Archimedean property {\itshape ad hoc}, we prefer to derive it from first principles by imposing additional axioms, besides coherence, on the sets of desirable option sets that model a subject's choices.

In order to achieve this, we strengthen our interpretation for a set of desirable option sets~\(\rejectset\). 
That is, for any \(\optset\in\optsets\), \(\optset\in\rejectset\) is henceforth taken to mean that there is at least one gamble \(f\) in~\(\optset\) that is \emph{strictly desirable}~\cite{walley1991,walley2000},\footnote{There may arise, due to Walley's \cite{walley1991,walley2000} perhaps unfortunate introduction of this terminology, some confusion in the reader's mind about the use of `strict'. In most accounts of preference relations, the term `strict preference' refers to `(weak) preference without indifference', and it is also in this sense that we have used the term `strict preference' in the Introduction. Walley uses the moniker `strict' for a stronger requirement, which is essentially based on some lower (or linear) prevision being \emph{strictly} positive. We maintain this rather unhappy use of terminology here merely for historical reasons, but insist on warning the reader about the possible confusion this may entail.} in the sense that there is some \(\epsilon\in\posreals\) such that \(f-\epsilon\) is desirable, where---as before for \(\mu\)---we identify the real number \(\epsilon\) with the constant gamble that takes the value \(\epsilon\).

In an earlier paper~\cite{debock2019:interpretation}, this interpretation in terms of strict desirability led us to propose a notion of Archimedeanity for sets of desirable option sets. 
With hindsight, we now prefer to call it \emph{strong Archimedeanity}, and to reserve the term Archimedeanity for sets of desirable option sets~\(\rejectset[\mathscr{P}]\) that correspond to a set of lower previsions~\(\mathscr{P}\subseteq\cohlowprevs\).

\begin{definition}[Strongly Archimedean set of desirable option sets]\label{def:gamblestronglyarchimedean:setsofsets}
We call a set of desirable option sets \(\rejectset\) \emph{strongly Archimedean} if it is coherent and satisfies
\begin{enumerate}[label=\(\mathrm{K}_{\mathrm{A}}\).,ref=\(\mathrm{K}_{\mathrm{A}}\),leftmargin=*]
\item\label{ax:rejects:gamblearch} for all \(\optset\in\rejectset\), there is some \(\epsilon\in\posreals\) such that \(\optset-\epsilon\in\rejectset\).
\end{enumerate}
\end{definition}
\noindent As we proved in earlier work~\cite{debock2019:interpretation}, strongly Archimedean choice models are in a one-to-one correspondence with sets of coherent lower previsions that are \emph{closed} in the topology on bounded real functionals induced by point-wise convergence.

\begin{theorem}[Representation for strongly Archimedean choice models]\label{theo:rejectsets:representation:lowerprev:twosided}
A set of desirable option sets \(\rejectset\) is strongly Archimedean if and only if there is a non-empty closed set \(\mathscr{P}\subseteq\cohlowprevs\) of coherent lower previsions such that \(\rejectset=\rejectset[\mathscr{P}]\).
The largest such set~\(\mathscr{P}\) is then\/ \(\cohlowprevs(\rejectset)\).
\end{theorem}
\begin{proof}
This result is a direct consequence of~\cite[Theorem~28 and Proposition~24]{debock2019:interpretation}.
\end{proof}

If we compare this result to Definition~\ref{def:archimedeanity}, we see that every strongly Archimedean set of desirable options is---as the terminology also suggests---Archimedean. 
The axiom of strong Archimedeanity can therefore be employed as a justification for working with an Archimedean choice model, or equivalently, with a set \(\mathscr{P}\) of coherent lower previsions. 
However, strong Archimedeanity is a bit too strong for that purpose---hence our change in terminology with respect to~\cite{debock2019:interpretation}---because it additionally implies that \(\mathscr{P}\) is closed.

In order to resolve this issue, one of us has developed alternative axioms that weaken strong Archimedeanity in such a way that the closedness condition in Theorem~\ref{theo:rejectsets:representation:lowerprev:twosided} is no longer needed, which makes sure that these alternative axioms characterise Archimedeanity~\cite{ipmu2020debock,ipmu2020debock:arxiv,partialorderchoice2020debock:arxiv}.
They are still based on an interpretation in terms of strict desirability, but employ this interpretation more subtly. 
Simply put, the subtlety involves the fact that even if a subject assesses that there is some \(f\in\optset\) that is strictly desirable, meaning that there is some \(\epsilon\in\posreals\) such that \(f-\epsilon\) is desirable, she may not know for which specific \(\epsilon\) this is the case. 
We are then no longer justified in stating that \(\optset-\epsilon\in\rejectset\) for some \(\epsilon\in\posreals\), as strong Archimedeanity does. 
Fortunately, however, we can still infer---more involved---conditions on \(\rejectset\) from such an assessment, which turn out to be equivalent to Archimedeanity.
A detailed exposition of these ideas falls outside the scope of the present discussion though; for more information on Archimedean choice functions and how to axiomatise them, we refer to our most recent work on this topic~\cite{ipmu2020debock,ipmu2020debock:arxiv,partialorderchoice2020debock:arxiv,ipmu2020decooman,ipmu2020decooman:arxiv}.
For our present purposes, it will suffice to merely remember that Archimedeanity can be given an axiomatic basis that motivates the use of general---not necessarily closed---sets of coherent lower previsions. 

In order to obtain a representation in terms of sets of linear rather than coherent lower previsions, it turns out that we need to impose one more axiom, which is due to Seidenfeld et al.~\cite{seidenfeld2010}.
It states that we can remove from a desirable option set those options that are positive linear combinations---mixtures---of a number of its other options.\footnote{In this sense, it would perhaps be preferable to call it an `unmixing property', as the term `mixing' is better suited for an approach that favours choice over rejection. Nevertheless, we have decided to stick to `mixing', for reasons of consistency with the terminology introduced in \cite{seidenfeld2010}.}
It involves the following closure operator, which adds to a set of options all the positive linear combinations of any finite number of its elements:
\begin{equation*}
\posi(\optset)\coloneqq
\cset[\bigg]{\sum_{k=1}^n\lambda_k\gbli[k]}{n\in\naturals,\lambda_k>0,\gbli[k]\in\optset}
\text{ for all \(\optset\subseteq\gbls\)}.
\end{equation*}

\begin{definition}[Mixing property for sets of desirable option sets]\label{def:mixingrejects}
We call a set of desirable option sets \(\rejectset\) \emph{mixing} if it is coherent and satisfies
\begin{enumerate}[label=\(\mathrm{K}_{\mathrm{M}}\).,ref=\(\mathrm{K}_{\mathrm{M}}\),leftmargin=*]
\item\label{ax:rejects:removepositivecombinations} if \(\altoptset\in\rejectset\) and \(\optset\subseteq\altoptset\subseteq\posi\group{\optset}\), then also \(\optset\in\rejectset\), for all \(\optset,\altoptset\in\optsets\).
\end{enumerate}
\end{definition}

\begin{proposition}[{\protect\cite[Proposition~2]{ipmu2020debock}}]
\label{prop:mixingimpliessetoflinear}
Let \(\rejectset\) be an Archimedean set of desirable option sets that is mixing. 
Then every coherent lower prevision~\(\lowprev\) in~\(\cohlowprevs(\rejectset)\) is linear.
\end{proposition}

\begin{proposition}
\label{prop:setoflinearimpliesmixing}
Let\/ \(\mathscr{P}\subseteq\linprevs\) be a non-empty set of linear previsions. 
Then \(\rejectset[\mathscr{P}]\) is mixing.
\end{proposition}

\begin{proof}
\label{theo:representationformixing}
Since mixingness is clearly preserved under taking (non-empty) intersections, it suffices to prove that for any \(\linprev\in\linprevs\), \(\rejectset[\linprev]\) is mixing. 
To this end, consider any \(\linprev\in\linprevs\) and any \(\optset,\altoptset\in\optsets\) such that \(\altoptset\in\rejectset[\linprev]\) and \(\optset\subseteq\altoptset\subseteq\posi\group{\optset}\). 
Since \(\altoptset\in\rejectset[\linprev]\), there is some \(\altgbl\in\altoptset\) such that \(\linprev(\altgbl)>0\). 
Since \(\altoptset\subseteq\posi\group{\optset}\), there are \(n\in\naturals\) and, for each \(k\in\set{1,\dots,n}\), \(\lambda_k>0\) and \(\gbli[k]\in\optset\) such that \(\altgbl=\sum_{k=1}^n\lambda_k\gbli[k]\). 
Hence, it follows from the linearity of~\(\linprev\) that
\begin{equation*}
0
<\linprev(\altgbl)
=\linprev\bigg(\sum_{k=1}^n\lambda_k\gbli[k]\bigg)
=\sum_{k=1}^n\lambda_k\linprev(\gbli[k]),
\end{equation*}
which implies that there is at least one \(k^*\in\set{1,\dots,n}\) such that \(\linprev(\gbli[k^*])>0\). 
Since \(\gbli[k^*]\in\optset\), this implies that \(\optset\in\rejectset[\linprev]\).
\end{proof}

By combining these two results with Theorem~\ref{theo:rejectsets:representation:lowerprev:twosided} and Definition~\ref{def:archimedeanity}, we obtain representation in terms of (closed) sets of linear previsions.

\begin{theorem}\label{thm:archimedean:and:mixing}
A set of desirable option sets \(\rejectset\) is Archimedean and mixing if and only if there is a non-empty set \(\mathscr{P}\subseteq\linprevs\) of coherent linear previsions such that \(\rejectset=\rejectset[\mathscr{P}]\). 
Similarly, \(\rejectset\) is strongly Archimedean and mixing if and only if there is a non-empty closed set \(\mathscr{P}\subseteq\linprevs\) of linear previsions such that \(\rejectset=\rejectset[\mathscr{P}]\).
In both cases, the largest such set~\(\mathscr{P}\) is \(\linprevs(\rejectset)\coloneqq\cset{\linprev\in\linprevs}{\rejectset\subseteq\rejectset[\linprev]}\), and \(\cohlowprevs(\rejectset)=\linprevs(\rejectset)\).
\end{theorem}

\begin{proof}
First assume that \(\rejectset\) is (strongly) Archimedean and mixing. 
Since \(\rejectset\) is (strongly) Archimedean, we know from Definition~\ref{def:archimedeanity} (Theorem~\ref{theo:rejectsets:representation:lowerprev:twosided}) that there is a non-empty (closed) set \(\mathscr{P}\subseteq\cohlowprevs\) of coherent lower previsions such that \(\rejectset=\rejectset[\mathscr{P}]\). 
Furthermore, since \(\rejectset\) is mixing, it follows from Proposition~\ref{prop:mixingimpliessetoflinear} that \(\mathscr{P}\subseteq\linprevs\).

Conversely, assume that there is a non-empty (closed) set \(\mathscr{P}\subseteq\linprevs\) such that \(\rejectset=\rejectset[\mathscr{P}]\).
Proposition~\ref{prop:setoflinearimpliesmixing} then implies that \(\rejectset\) is mixing. 
Furthermore, since \(\linprevs\subseteq\cohlowprevs\), it follows from Definition~\ref{def:archimedeanity} (Theorem~\ref{theo:rejectsets:representation:lowerprev:twosided}) that \(\rejectset\) is (strongly) Archimedean. 
It remains to show that any such \(\mathscr{P}\) is a subset of~\(\linprevs(\rejectset)\), that \(\linprevs(\rejectset)\) is non-empty (and closed), that \(\linprevs(\rejectset)=\cohlowprevs(\rejectset)\) and that \(\rejectset=\rejectset[\linprevs(\rejectset)]\).

That \(\mathscr{P}\) is a subset of~\(\linprevs(\rejectset)\) follows immediately from the definitions of~\(\rejectset[\mathscr{P}]\) and~\(\linprevs(\rejectset)\) and from the fact that \(\rejectset=\rejectset[\mathscr{P}]\). 
On the one hand, since \(\mathscr{P}\) is non-empty, this establishes the non-emptiness of~\(\linprevs(\rejectsets)\). 
On the other hand, this implies that \(\rejectset[\linprevs(\rejectset)]\subseteq\rejectset[\mathscr{P}]=\rejectset\), and therefore, since \(\rejectset\) is clearly a subset of~\(\rejectset[\linprevs(\rejectset)]\), that \(\rejectset=\rejectset[\linprevs(\rejectset)]\). 
That \(\linprevs(\rejectset)=\cohlowprevs(\rejectset)\) follows from Proposition~\ref{prop:mixingimpliessetoflinear}. 
That \(\linprevs(\rejectset)\) is closed if \(\rejectset\) is strongly Archimedean and mixing, finally, follows from the fact that \(\linprevs(\rejectset)=\cohlowprevs(\rejectset)\) and Theorem~\ref{theo:rejectsets:representation:lowerprev:twosided}.
\end{proof}

\section{S-Irrelevance for events}\label{sec:S-irrelevance:for:events}
Let us now assume that our subject's preferences are modelled by some coherent set of desirable option sets \(\rejectset\).
We consider two events \(\aevt\) and \(\bevt\), and investigate the import of Teddy Seidenfeld's operationalisation of the independence requirement discussed in the Introduction: for any \(\gbl\) and \(\altgbl\) in~\(\gbls[\bevt]\) and all real \(\epsilon>0\) 
\begin{equation*}
\ind{\aevt}\gbl+\ind{\co{\aevt}}\altgbl-\epsilon
\text{ is rejected from the option set }
\optset[\epsilon]\coloneqq\set{\gbl,\altgbl,\ind{\aevt}f+\ind{\co{\aevt}}\altgbl-\epsilon}.
\end{equation*}
Taking into account that
\begin{equation*}
\gbl-(\ind{\aevt}\gbl+\ind{\co{\aevt}}\altgbl-\epsilon)
=(\gbl-\altgbl)\ind{\co{\aevt}}+\epsilon
\text{ and }
\altgbl-(\ind{\aevt}\gbl+\ind{\co{\aevt}}\altgbl-\epsilon)
=(\altgbl-\gbl)\ind{\aevt}+\epsilon,
\end{equation*}
Equation~\eqref{eq:choiceandrejectionfromK} leads us to the following definition.
It acknowledges that the above requirement is not necessarily symmetrical in~\(\aevt\) and \(\bevt\), and therefore leads to a notion of irrelevance; independence is then introduced as symmetrised irrelevance.

\begin{definition}[S-irrelevance and S-independence]\label{def:S-irrelevance}
Consider two events \(\aevt,\bevt\subseteq\pspace\).
We say that \(\aevt\) is \emph{S-irrelevant} to \(\bevt\) with respect to a coherent set of desirable option sets~\(\rejectset\) if
\begin{equation}\label{eq:simplifiedphysirr}
\set{\ind{\co{\aevt}}\gbl+\epsilon,-\ind{\aevt}\gbl+\epsilon}\in\rejectset
\text{ for all \(\gbl\in\gbls[\bevt]\) and all \(\epsilon>0\)}.
\end{equation}
We say that \(\aevt\) and \(\bevt\) are \emph{S-independent} with respect to~\(\rejectset\) if \(\aevt\) is S-irrelevant to~\(\bevt\) and \(\bevt\) is S-irrelevant to~\(\aevt\).
\end{definition}
\noindent
These irrelevance and independence notions are invariant under complementation, as follows easily from Definition~\ref{def:S-irrelevance}.

\begin{proposition}\label{prop:symmetry:for:S-irrelevance}
Consider any two events \(\aevt,\bevt\subseteq\pspace\), and any \(\taevt\in\set{\aevt,\co{\aevt}}\) and \(\tbevt\in\set{\bevt,\co{\bevt}}\).
Then \(\aevt\) is S-irrelevant to~\(\bevt\) with respect to a coherent set of desirable option sets~\(\rejectset\) if and only if \(\taevt\) is S-irrelevant to~\(\tbevt\) with respect to~\(\rejectset\).
Similarly, \(\aevt\) and \(\bevt\) are S-independent with respect to~\(\rejectset\) if and only if \(\taevt\) and~\(\tbevt\) are.
\end{proposition}

\begin{proof}
We concentrate on the statement concerning S-irrelevance, as the proof for the statement about S-independence then follows immediately. 
Due to the symmetry of the statement, it clearly suffices to prove necessity.

So assume that \(\aevt\) is S-irrelevant to~\(\bevt\) with respect to~\(\rejectset\). 
We start by observing that \(\gbls[\tbevt]=\gbls[\bevt]\). 
So if \(\taevt=\aevt\), our assumption trivially implies that \(\taevt\) is S-irrelevant to~\(\tbevt\) with respect to~\(\rejectset\). 
If \(\taevt=\co{\aevt}\), then for any \(\gbl\in\gbls[\tbevt]\) and \(\epsilon>0\), since \(-\gbl\in\gbls[\tbevt]=\gbls[\bevt]\), it follows from the assumption that
\begin{equation*}
\set{-\ind{\taevt}\gbl+\epsilon,\ind{\co{\taevt}}\gbl+\epsilon}
=
\set{\ind{\co{\aevt}}(-\gbl)+\epsilon,-\ind{\aevt}(-\gbl)+\epsilon}\in\rejectset.
\end{equation*}
Hence, also in this case, \(\taevt\) is S-irrelevant to~\(\tbevt\) with respect to~\(\rejectset\).
\end{proof}

If the coherent~\(\rejectset\) is moreover Archimedean, as we will henceforth typically assume, then S-irrelevance with respect to~\(\rejectset\) can be characterised more simply in terms of S-irrelevance with respect to specific binary choice models of the types~\(\rejectset[\lowprev]\) and~\(\rejectset[\linprev]\).
Similar statements will of course hold for S-independence.

\begin{proposition}\label{prop:S-irrelevance:in:terms_of:binary}
Consider any events \(\aevt,\bevt\subseteq\pspace\) and a set of desirable option sets~\(\rejectset\).
If \(\rejectset\) is Archimedean, then \(\aevt\) is S-irrelevant to~\(\bevt\) with respect to~\(\rejectset\) if and only if \(\aevt\) is S-irrelevant to~\(\bevt\) with respect to~\(\rejectset[\lowprev]\) for all \(\lowprev\in\cohlowprevs(\rejectset)\).
Similarly, if \(\rejectset\) is Archimedean and mixing, then \(\aevt\) is S-irrelevant to~\(\bevt\) with respect to~\(\rejectset\) if and only if \(\aevt\) is S-irrelevant to~\(\bevt\) with respect to~\(\rejectset[\linprev]\) for all \(\linprev\in\linprevs(\rejectset)\).
\end{proposition}

\begin{proof}
If we combine Definitions~\ref{def:archimedeanity} and~\ref{def:S-irrelevance}, we see that \(\aevt\) is S-irrelevant to~\(\bevt\) with respect to an Archimedean~\(\rejectset\) if and only if
\begin{equation*}
\set{\ind{\co{\aevt}}\gbl+\epsilon,-\ind{\aevt}\gbl+\epsilon}\in\bigcap\cset{\rejectset[\lowprev]}{\lowprev\in\cohlowprevs(\rejectset)}
\text{ for all \(\gbl\in\gbls[\bevt]\) and all \(\epsilon>0\)},
\end{equation*}
which proves the statement for Archimedeanity.
If we also impose mixingness on~\(\rejectset\), the second statement follows at once from the first and the fact that then, according to Theorem~\ref{thm:archimedean:and:mixing}, \(\cohlowprevs(\rejectset)=\linprevs(\rejectset)\).
\end{proof}
\noindent It therefore behoves us to study S-irrelevance with respect to such special binary models \(\rejectset[\linprev]\) and \(\rejectset[\lowprev]\).

\subsection{S-irrelevance with respect to linear prevision models}
First, we consider any linear prevision~\(\linprev\) on~\(\gbls\) and study the implications of S-irrelevance with respect to the binary choice model \(\rejectset[\linprev]\).
For any events \(\aevt\) and \(\bevt\), Equation~\eqref{eq:Kfromlinprev} then clearly implies that \(\aevt\) is S-irrelevant to~\(\bevt\) if and only if
\begin{equation*}
\linprev(\ind{\co{\aevt}}\gbl+\epsilon)>0\text{ or }\linprev(-\ind{\aevt}\gbl+\epsilon)>0
\text{ for all \(\gbl\in\gbls[\bevt]\setminus\set{0}\) and all \(\epsilon>0\)}, 
\end{equation*} 
or equivalently, since \(\linprev\) is constant additive [\ref{ax:lowprev:constantadditive}],
\begin{equation}\label{eq:S-irrelevance:linprev}
\linprev(\ind{\co{\aevt}}\gbl)\geq0\text{ or }\linprev(-\ind{\aevt}\gbl)\geq0
\text{ for all \(\gbl\in\gbls[\bevt]\)}.
\end{equation}

We now set out to investigate how this specific instance of S-irrelevance relates to the usual independence condition.
We will find that for this particular type of binary choice model, S-irrelevance, S-independence and the usual independence notion coincide.

\begin{definition}[Independent events with respect to a linear prevision]
We will call any two events \(\aevt\) and \(\bevt\) \emph{independent} with respect to a linear prevision~\(\linprev\) on \(\gbls\) if \(\linprev(\aevt\cap\bevt)=\linprev(\aevt)\linprev(\bevt)\). 
\end{definition}

\begin{proposition}\label{prop:symmetry:for:independence}
Consider any two events \(\aevt,\bevt\subseteq\pspace\), and consider any \(\taevt\in\set{\aevt,\co{\aevt}}\) and \(\tbevt\in\set{\bevt,\co{\bevt}}\).
Then \(\aevt\) and \(\bevt\) are independent with respect to a linear prevision~\(\linprev\) on \(\gbls\) if and only if \(\taevt\) and \(\tbevt\) are.
\end{proposition}

\begin{proof}
It clearly suffices to assume that \(\aevt\) and \(\bevt\) are independent, and to prove that \(\co{\aevt}\) and \(\bevt\) are.
So assume that \(\aevt\) and \(\bevt\) are independent. 
Then
\begin{align*}
\linprev(\co{\aevt})\linprev(\bevt)
=[1-\linprev(\aevt)]\linprev(\bevt)
&=\linprev(\bevt)-\linprev(\aevt)\linprev(\bevt)
=\linprev(\bevt)-\linprev(\aevt\cap\bevt)\\
&=\linprev(\bevt\setminus(\aevt\cap\bevt))
=\linprev(\bevt\cap\co{\aevt}),
\end{align*}
implying that \(\co{\aevt}\) and \(\bevt\) are independent as well.
\end{proof}

\begin{theorem}\label{theo:lineartwosided}
Consider any two events \(\aevt,\bevt\subseteq\pspace\), and any linear prevision~\(\linprev\) on~\(\gbls\).
Then the following statements are equivalent: 
\begin{enumerate}[label=\upshape(\roman*),leftmargin=*]
\item \(\aevt\) is S-irrelevant to~\(\bevt\) with respect to~\(\rejectset[\linprev]\);
\item \(\aevt\) and \(\bevt\) are S-independent with respect to~\(\rejectset[\linprev]\);
\item \(\aevt\) and \(\bevt\) are independent with respect to~\(\linprev\).
\end{enumerate}
\end{theorem}

\begin{proof} 
Since \(\aevt\) is independent of~\(\bevt\) if and only if \(\bevt\) is independent of~\(\aevt\), it clearly suffices to prove the equivalence of S-irrelevance and independence.

We begin with the `only if' part.
Assume \emph{ex absurdo} that \(\aevt\) and \(\bevt\) are not independent with respect to~\(\linprev\), so \(\linprev(\aevt\cap\bevt)\neq\linprev(\aevt)\linprev(\bevt)\). 
We will assume that \(\linprev(\aevt\cap\bevt)>\linprev(\aevt)\linprev(\bevt)\), but the arguments are completely analogous for the case that \(\linprev(\aevt\cap\bevt)<\linprev(\aevt)\linprev(\bevt)\): simply reverse the roles of~\(\aevt\) and \(\co{\aevt}\).

Let \(\delta\coloneqq\linprev(\aevt\cap\bevt)-\linprev(\aevt)\linprev(\bevt)>0\).
Then
\begin{align*}
\linprev(\co{\aevt}\cap\bevt)-\linprev(\co{\aevt})\linprev(\bevt)
&=\linprev(\bevt)-\linprev(\aevt\cap\bevt)-\linprev(\co{\aevt})\linprev(\bevt)\\
&=\linprev(\aevt)\linprev(\bevt)-\linprev(\aevt\cap\bevt)=-\delta<0,
\end{align*}
so \(\linprev(\co{\aevt}\cap\bevt)<\linprev(\co{\aevt})\linprev(\bevt)\).
Now let \(\gbl\coloneqq\ind{\bevt}-\linprev(\bevt)\in\gbls[\bevt]\), then
\begin{equation*}
\linprev(\ind{\co{\aevt}}\gbl)
=\linprev(\ind{\co{\aevt}}\ind{\bevt})-\linprev(\bevt)\linprev(\ind{\co{\aevt}})
=\linprev(\co{\aevt}\cap\bevt)-\linprev(\co{\aevt})\linprev(\bevt)=-\delta<0
\end{equation*}
and 
\begin{equation*}
\linprev(-\ind{\aevt}\gbl)
=-\linprev(\ind{\aevt}\ind{\bevt})+\linprev(\bevt)\linprev(\ind{\aevt})
=\linprev(\aevt)\linprev(\bevt)-\linprev(\aevt\cap\bevt)=-\delta<0.
\end{equation*}
It therefore follows from the S-irrelevance criterion~\eqref{eq:S-irrelevance:linprev} that \(\aevt\) is not S-irrelevant to~\(\bevt\) with respect to~\(\rejectset[\linprev]\), a contradiction. 
We conclude that \(\aevt\) and \(\bevt\) are independent with respect to~\(\linprev\).

It remains to prove the `if' part, so let us assume that \(\aevt\) and \(\bevt\) are independent with respect to~\(\linprev\). 
We will use Definition~\ref{def:S-irrelevance} to prove that \(\aevt\) is S-irrelevant to~\(\bevt\) with respect to~\(\rejectset[\linprev]\). 
Fix any \(\gbl\in\gbls[\bevt]\), then we need to show that the S-irrelevance criterion~\eqref{eq:S-irrelevance:linprev} is satisfied. 
Since \(\aevt\) and \(\bevt\) are independent with respect to~\(\linprev\)---and also using Proposition~\ref{prop:symmetry:for:independence}---this criterion simplifies to
\begin{equation*} 
\linprev(\co{\aevt})\linprev(\gbl)\geq0
\text{ or }
-\linprev(\aevt)\linprev(\gbl)\geq0,
\end{equation*}
which is clearly always satisfied, because the first inequality holds when \(\linprev(\gbl)\geq0\) and the second one holds when \(\linprev(\gbl)\leq0\).
\end{proof}
Since we have seen that sets of desirable option sets~\(\rejectset\) that are Archimedean and mixing are completely determined by the linear prevision models \(\rejectset[\linprev]\) that include them, it ought not to surprise us, in view of Theorem~\ref{theo:lineartwosided}, that for such \(\rejectset\), we can reduce S-irrelevance (and S-independence) to independence with respect to their representing linear previsions~\(\linprev\).

\begin{theorem}
Let \(\rejectset\) be an Archimedean and mixing set of desirable option sets, and consider any events~\(\aevt,\bevt\subseteq\pspace\). 
Then the following statements are equivalent: 
\begin{enumerate}[label=\upshape(\roman*),leftmargin=*]
\item \(\aevt\) is S-irrelevant to~\(\bevt\) with respect to~\(\rejectset\);
\item \(\aevt\) and \(\bevt\) are S-independent with respect to~\(\rejectset\);
\item \(\aevt\) and \(\bevt\) are independent with respect to~\(\linprev\), for all \(\linprev\in\linprevs(\rejectset)\).
\end{enumerate}
\end{theorem}

\begin{proof}
This is an immediate consequence of Proposition~\ref{prop:S-irrelevance:in:terms_of:binary} and Theorem~\ref{theo:lineartwosided}.
\end{proof}

\subsection{S-irrelevance with respect to lower prevision models}\label{sec:lowerprevisions:events}
Next, we consider a coherent lower prevision~\(\lowprev\) on~\(\gbls\) and study the implications of S-irrelevance with respect to the binary choice model \(\rejectset[\lowprev]\). 
For any two events~\(\aevt\) and~\(\bevt\), Equation~\eqref{eq:Kfromlowprev} then clearly implies that \(\aevt\) is S-irrelevant to~\(\bevt\) with respect to~\(\rejectset[\lowprev]\) if and only if
\begin{equation*}
\lowprev(\ind{\co{\aevt}}\gbl+\epsilon)>0\text{ or }\lowprev(-\ind{\aevt}\gbl+\epsilon)>0
\text{ for all \(\gbl\in\gbls[\bevt]\) and all \(\epsilon>0\)}, 
\end{equation*} 
or equivalently, since \(\lowprev\) is constant additive~[\ref{ax:lowprev:constantadditive}],
\begin{equation}\label{eq:S-irrelevance:lowprev}
\lowprev(\ind{\co{\aevt}}\gbl)\geq0\text{ or }\lowprev(-\ind{\aevt}\gbl)\geq0
\text{ for all \(\gbl\in\gbls[\bevt]\)}.
\end{equation} 

Here too, we set out to investigate how this particular instance of S-irrelevance relates to the usual independence condition. 
A first important result is that in the context of lower previsions, S-irrelevance implies independence in the usual sense for every dominating linear prevision.

\begin{proposition}\label{prop:loweronesided}
Consider any two events \(\aevt,\bevt\subseteq\pspace\) and any coherent lower prevision~\(\lowprev\) on~\(\gbls\), and let \(\aevt\) be S-irrelevant to~\(\bevt\) with respect to~\(\rejectset[\lowprev]\). 
Then \(\aevt\) and \(\bevt\) are independent with respect to all \(\linprev\in\linprevs(\lowprev)\).
\end{proposition}

\begin{proof}
Assume \emph{ex absurdo} that there is some \(\linprev\in\linprevs(\lowprev)\) with respect to which \(\aevt\) and \(\bevt\) are not independent. 
It then follows from Theorem~\ref{theo:lineartwosided} that, with respect to the corresponding \(\rejectset[\linprev]\), \(\aevt\) is not S-irrelevant to~\(\bevt\), implying that there is some \(\gbl\in\gbls[\bevt]\) such that \(\linprev(\ind{\co{\aevt}}\gbl)<0\) and \(\linprev(-\ind{\aevt}\gbl)<0\).
Since \(\linprev\) dominates \(\lowprev\), this immediately implies that also \(\lowprev(\ind{\co{\aevt}}\gbl)<0\) and \(\lowprev(-\ind{\aevt}\gbl)<0\).
Hence, \(\aevt\) is not S-irrelevant to~\(\bevt\) with respect to~\(\rejectset[\lowprev]\), a contradiction.
\end{proof}

Given this result, and drawing inspiration from Theorem~\ref{theo:lineartwosided}, one might be led to think that independence of the dominating linear previsions is not only necessary, but also sufficient for S-irrelevance. 
This is not the case, though; the non-linearity of \(\lowprev\) makes for a slightly more involved picture. 
To arrive at a condition that is both necessary and sufficient, we start by introducing a notion of triviality for events: we say that an event \(\aevt\) is \emph{trivial} with respect to~\(\lowprev\) if (either) \(\uppprev(\aevt)=0\) or \(\uppprev(\co{\aevt})=0\). 
The following result shows that this notion of triviality yields a sufficient condition for S-independence, and hence also S-irrelevance.

\begin{proposition}\label{prop:triviality:implies:S:independence}
If the events \(\aevt\) or \(\bevt\) are trivial with respect to the coherent lower prevision~\(\lowprev\) on~\(\gbls\), then \(\aevt\) and \(\bevt\) are S-independent with respect to~\(\rejectset[\lowprev]\).
\end{proposition}
\noindent
Our proof for this result, as well as that of Theorem~\ref{theo:lower:equivalentdefinitions} further on, makes use of the following simple technical lemma, which basically states that the lower and upper prevision of a gamble don't depend on the values of that gamble on any event with zero upper probability.

\begin{lemma}\label{lem:upper:probability:zero}
Consider any event \(\cevt\subseteq\pspace\) such that \(\uppprev(\cevt)=0\).
Then \(\lowprev(\altgbl\ind{\cevt}+\altgbltoo)=\lowprev(\altgbltoo)\) and \(\uppprev(\altgbl\ind{\cevt}+\altgbltoo)=\uppprev(\altgbltoo)\) for all gambles \(\altgbl\) and \(\altgbltoo\) on~\(\pspace\).    
\end{lemma}

\begin{proof}
It follows from conjugacy that we need only prove the equality for lower previsions. To see that it holds for lower previsions, observe that the inequality \(\ind{\cevt}\inf\altgbl\leq\altgbl\ind{\cevt}\leq\ind{\cevt}\sup\altgbl\) and the coherence of~\(\lowprev\) [use~\ref{ax:lowprev:mixed:additivity} and~\ref{ax:lowprev:monotonicity}] indeed lead to
\begin{equation*}
\lowprev(\altgbltoo)
=\lowprev(\ind{\cevt}\inf\altgbl)+\lowprev(\altgbltoo)
\leq\lowprev(\altgbl\ind{\cevt}+\altgbltoo)
\leq\uppprev(\ind{\cevt}\sup\altgbl)+\lowprev(\altgbltoo)
=\lowprev(\altgbltoo),
\end{equation*}
where the equalities hold because for all real \(\lambda\), also by coherence [use~\ref{ax:lowprev:homo}], \(\lowprev(\lambda\ind{\cevt})=\uppprev(\lambda\ind{\cevt})=0\) whenever \(\lowprev(\cevt)=\uppprev(\cevt)=0\) [that \(\lowprev(\cevt)=0\) follows from \(\uppprev(\cevt)=0\) and~\ref{ax:lowprev:infandsup}].
\end{proof}

\begin{proof}\hspace{-1pt}{\bf{\itshape{of Proposition~\ref{prop:triviality:implies:S:independence}}}}
Assume that \(\aevt\) or \(\bevt\) are trivial with respect to~\(\lowprev\), then it suffices to prove that \(\aevt\) is S-irrelevant to~\(\bevt\).
So consider any \(\gbl\coloneqq\lambda\ind{\bevt}+\mu\ind{\co{\bevt}}\), with \((\lambda,\mu)\in\reals^2\).
Then we have to show that \(\lowprev(\ind{\aevt}\gbl)\geq0\) or \(\lowprev(-\ind{\co{\aevt}}\gbl)\geq0\).
Due to the symmetry, it suffices to consider the following two possible cases.

The first case is that \(\uppprev(\aevt)=0\).
Lemma~\ref{lem:upper:probability:zero} then guarantees that \(\lowprev(\ind{E}\gbl)=\lowprev(0)=0\) [use~\ref{ax:lowprev:infandsup}].

The second case is that \(\uppprev(\bevt)=0\).
Considering that
\begin{equation*}
\ind{\aevt}\gbl=\lambda\ind{\aevt}\ind{\bevt}+\mu\ind{\aevt}\ind{\co{\bevt}}
\text{ and }
-\ind{\co{\aevt}}\gbl=-\lambda\ind{\co{\aevt}}\ind{\bevt}-\mu\ind{\co{\aevt}}\ind{\co{\bevt}},
\end{equation*}
we now infer from Lemma~\ref{lem:upper:probability:zero} that
\begin{equation*}
\lowprev(\ind{\aevt}\gbl)=\lowprev(\mu\ind{\aevt}\ind{\co{\bevt}})
\text{ and }
\lowprev(-\ind{\co{\aevt}}\gbl)=\lowprev(-\mu\ind{\co{\aevt}}\ind{\co{\bevt}}).
\end{equation*}
There are now two possibilities.
If \(\mu\geq0\), then we infer [use~\ref{ax:lowprev:homo} and~\ref{ax:lowprev:infandsup}] from the first equality that \(\lowprev(\ind{\aevt}\gbl)=\lowprev(\mu\ind{\aevt}\ind{\co{\bevt}})=\mu\lowprev(\ind{\aevt}\ind{\co{\bevt}})\geq0\).
If \(\mu\leq0\), then we infer [again use~\ref{ax:lowprev:homo} and~\ref{ax:lowprev:infandsup}] from the second equality that \(\lowprev(-\ind{\co{\aevt}}\gbl)=\lowprev(-\mu\ind{\co{\aevt}}\ind{\co{\bevt}})=(-\mu)\lowprev(\ind{\co{\aevt}}\ind{\co{\bevt}})\geq0\).
\end{proof}

So triviality is a sufficient condition for S-irrelevance and S-independence.
The following crucial proposition shows, on the other hand, that whenever \(\aevt\) is not trivial, S-irrelevance of~\(\aevt\) to~\(\bevt\) implies precision for gambles on~\(\bevt\).

\begin{proposition}\label{prop:lowerprecise}
Consider any two events \(\aevt,\bevt\subseteq\pspace\) and any coherent lower prevision~\(\lowprev\) on~\(\gbls\).
If \(\aevt\) is S-irrelevant to~\(\bevt\) with respect to~\(\rejectset[\lowprev]\), then \(\aevt\) is trivial or
\begin{equation*}
\lowprev(\bevt)=\uppprev(\bevt)
\text{ and hence also }
\lowprev(\gbl)=\uppprev(\gbl)\,\text{ for all \(\gbl\in\gbls[\bevt]\).}
\end{equation*}
\end{proposition}
\noindent
Our proof for this result makes use of two technical lemmas. 
The first of these two is fairly straightforward; it establishes that if the probability of an event is precise---meaning that its lower and upper probabilities coincide---then every gamble on the occurrence of this event has a precise prevision---meaning that its lower and upper prevision coincide.

\begin{lemma}\label{lem:Fpreciseimpliesgambleprecise}
Consider any event \(\bevt\subseteq\pspace\) and any coherent lower prevision~\(\lowprev\) on~\(\gbls\). 
If \(\smash{\lowprev(\bevt)=\uppprev(\bevt)}\), then also \(\smash{\lowprev(\gbl)=\uppprev(\gbl)}\) for all \(\gbl\in\gbls[\bevt]\).
\end{lemma}

\begin{proof}
Since \(\gbl\in\gbls[\bevt]\), we know that there are \(\lambda,\mu\in\reals\) such that \(\gbl=\lambda\ind{\bevt}+\mu\ind{\co{\bevt}}\).
Let \(\smash{\alpha\coloneqq\lowprev(\bevt)=\uppprev(\bevt)}\). 
For all \(\linprev\in\linprevs(\lowprev)\), we then have that
\begin{equation*}
\linprev(\gbl)
=\linprev(\lambda\ind{\bevt}+\mu\ind{\co{\bevt}})
=\lambda\linprev(\ind{\bevt})+\mu\linprev(\ind{\co{\bevt}})
=\lambda\linprev(\bevt)+\mu\linprev(\co{\bevt})
=\lambda\alpha+\mu(1-\alpha).
\end{equation*}
The result therefore follows from Equation~\eqref{eq:lowerandupperenvelope}.
\end{proof}
\noindent 
The second lemma on which our proof for Proposition~\ref{prop:lowerprecise} depends, is more involved and less intuitive. 
It states that if \(\aevt\) is S-irrelevant to \(\bevt\) for some lower prevision, then any two linear previsions that dominate this lower prevision but disagree on the probability of \(\bevt\) must either both assign zero probability to \(\aevt\), or both assign zero probability to \(\co{\aevt}\).

\begin{lemma}\label{lem:atleastonezero}
Consider any two events \(\aevt,\bevt\subseteq\pspace\) and any coherent lower prevision~\(\lowprev\) on~\(\gbls\). 
Let \(\aevt\) be S-irrelevant to~\(\bevt\) with respect to~\(\rejectset[\lowprev]\).
Then for any \(\linprev[1],\linprev[2]\in\linprevs(\lowprev)\) such that \(\linprev[1](\bevt)\neq\linprev[2](\bevt)\), there is some \(\taevt\in\set{\aevt,\co{\aevt}}\) such that \(\linprev[1](\taevt)=\linprev[2](\taevt)=0\).
\end{lemma}

\begin{proof}
Recall from Proposition~\ref{prop:loweronesided} that the assumptions imply in particular that \(\aevt\) and~\(\bevt\) are independent with respect to both \(\linprev[1]\) and \(\linprev[2]\).

We will assume without loss of generality that \(\linprev[1](\bevt)<\linprev[2](\bevt)\), and therefore that \(\delta\coloneqq(\linprev[2](\bevt)-\linprev[1](\bevt))/2>0\). 
In the other case, simply reverse the roles of~\(\linprev[1]\) and~\(\linprev[2]\) and repeat the argument.

If we let \(\kappa\coloneqq(\linprev[1](\bevt)+\linprev[2](\bevt))/2\), then \(\linprev[1](\bevt)<\kappa<\linprev[2](\bevt)\).
So if we now let \(\gbl\coloneqq\ind{\bevt}-\kappa\in\gbls[\bevt]\), then \(\linprev[1](\gbl)=-\delta<0<\delta=\linprev[2](\gbl)\). 

First, assume \emph{ex absurdo} that \(\linprev[1](\co{\aevt})>0\) and \(\linprev[2](\aevt)>0\). 
Then
\begin{equation*}
\linprev[1](\ind{\co{\aevt}}\gbl)
=\linprev[1](\co{\aevt}\cap\bevt)-\kappa\linprev[1](\co{\aevt})
=\linprev[1](\co{\aevt})\linprev[1](\bevt)-\kappa\linprev[1](\co{\aevt})
=-\delta\linprev[1](\co{\aevt})<0
\end{equation*}
and similarly
\begin{equation*}
\linprev[2](-\ind{\aevt}\gbl)
=-\linprev[2](\aevt\cap\bevt)+\kappa\linprev[2](\aevt)
=-\linprev[2](\aevt)\linprev[2](\bevt)+\kappa\linprev[2](\aevt)
=-\delta\linprev[2](\aevt)<0,
\end{equation*}
where we have used the independence of~\(\aevt\) and \(\bevt\) with respect to both \(\linprev[1]\) and~\(\linprev[2]\), together with Proposition~\ref{prop:symmetry:for:independence} for the second case.
This implies that also
\begin{equation*}
\lowprev(\ind{\co{\aevt}}\gbl)<0
\text{ and }
\lowprev(-\ind{\aevt}\gbl)<0,
\end{equation*}
contradicting the assumption that \(\aevt\) is S-irrelevant to~\(\bevt\) with respect to~\(\rejectset[\lowprev]\); see criterion~\eqref{eq:S-irrelevance:lowprev}. 
We therefore conclude that \(\linprev[1](\co{\aevt})=0\) or \(\linprev[2](\aevt)=0\).

Similarly, assume \emph{ex absurdo} that \(\linprev[1](\aevt)>0\) and \(\linprev[2](\co{\aevt})>0\).
Then
\begin{equation*}
\linprev[1](\ind{\aevt}\gbl)
=\linprev[1](\aevt\cap\bevt)-\kappa\linprev[1](\aevt)
=\linprev[1](\aevt)\linprev[1](\bevt)-\kappa\linprev[1](\aevt)
=-\delta\linprev[1](\aevt)<0
\end{equation*}
and also
\begin{equation*}
\linprev[2](-\ind{\co{\aevt}}\gbl)
=-\linprev[2](\co{\aevt}\cap\bevt)+\kappa\linprev[2](\co{\aevt})
=-\linprev[2](\co{\aevt})\linprev[2](\bevt)+\kappa\linprev[2](\co{\aevt})
=-\delta\linprev[2](\co{\aevt})<0,
\end{equation*}
where we have again used the independence of~\(\aevt\) and \(\bevt\) with respect to both \(\linprev[1]\) and~\(\linprev[2]\), together with Proposition~\ref{prop:symmetry:for:independence} for the second case.
It follows that 
\begin{equation*}
\lowprev(\ind{\aevt}\gbl)<0
\text{ and }
\lowprev(-\ind{\co{\aevt}}\gbl)<0,
\end{equation*}
contradicting the assumption that \(\aevt\) is S-irrelevant to~\(\bevt\) with respect to~\(\rejectset[\lowprev]\); see criterion~\eqref{eq:S-irrelevance:lowprev} for the gamble \(-\gbl\) this time. 
We therefore conclude that \(\linprev[1](\aevt)=0\) or \(\linprev[2](\co{\aevt})=0\).

In summary, we have so far found that
\begin{equation*}
(\text{\(\linprev[1](\co{\aevt})=0\) or \(\linprev[2](\aevt)=0\)})
\text{ and }
(\text{\(\linprev[1](\aevt)=0\) or \(\linprev[2](\co{\aevt})=0\)}).
\end{equation*}
After applying the distributivity of `and' over `or', and removing two contradictions, we see that this is equivalent to
\begin{equation*}
(\text{\(\linprev[1](\co{\aevt})=0\) and \(\linprev[2](\co{\aevt})=0\)})
\text{ or }
(\text{\(\linprev[1](\aevt)=0\) and \(\linprev[2](\aevt)=0\)}),
\end{equation*}
as claimed in the statement of the lemma.
\end{proof}

\begin{proof}\hspace{-1pt}{\bf{\itshape{of Proposition~\ref{prop:lowerprecise}}}}
Assume that \(\aevt\) is S-irrelevant to~\(\bevt\) with respect to~\(\smash{\rejectset[\lowprev]}\) and that \(\lowprev(\bevt)<\uppprev(\bevt)\). 
We prove that this implies that either \(\smash{\uppprev(\aevt)=0}\) or \(\smash{\uppprev(\co{\aevt})=0}\). 
The rest of the statement then follows from Lemma~\ref{lem:Fpreciseimpliesgambleprecise}.

Since \(\lowprev(\bevt)<\uppprev(\bevt)\), there are \(\linprev[1],\linprev[2]\in\linprevs(\lowprev)\) such that \(\linprev[1](\bevt)\neq\linprev[2](\bevt)\). 
Since \(\aevt\) is S-irrelevant to~\(\bevt\) with respect to~\(\rejectset[\lowprev]\), we know from Lemma~\ref{lem:atleastonezero} that there is some \(\taevt\in\set{\aevt,\co{\aevt}}\) such that \(\linprev[1](\taevt)=\linprev[2](\taevt)=0\).

Consider now any \(\linprev\in\linprevs(\lowprev)\). 
Since \(\linprev[1](\bevt)\neq\linprev[2](\bevt)\), there is at least one \(k\in\set{1,2}\) such that \(\linprev[k](\bevt)\neq\linprev(\bevt)\). 
Since \(\linprev,\linprev[k]\in\linprevs(\lowprev)\) and since \(\aevt\) is S-irrelevant to~\(\bevt\) with respect to~\(\smash{\rejectset[\lowprev]}\), Lemma~\ref{lem:atleastonezero} tells us that there is some \(\taevt[k]\in\set{\aevt,\co{\aevt}}\) such that \(\linprev(\taevt[k])=\linprev[k](\taevt[k])=0\). 
Assume \emph{ex absurdo} that \(\taevt[k]\neq\taevt\). 
Since \(\taevt[k]\) and \(\taevt\) belong to \(\set{\aevt,\co{\aevt}}\) and both \(\linprev[k](\taevt)=0\) and \(\linprev[k](\taevt[k])=0\), this would imply that \(\linprev[k](\aevt)=\linprev[k](\co{\aevt})=0\), a contradiction because these probabilities must sum to one.
So we find that \(\taevt[k]=\taevt\), and therefore also that \(\linprev(\taevt)=\linprev(\taevt[k])=0\). 
Since this is true for any \(\linprev\in\linprevs(\lowprev)\), we conclude that \(\uppprev(\taevt)=0\).
\end{proof}

We can now combine Propositions~\ref{prop:loweronesided}, \ref{prop:triviality:implies:S:independence} and~\ref{prop:lowerprecise} into the following result, which is the counterpart of Theorem~\ref{theo:lineartwosided} for lower prevision models. 
It yields a necessary and sufficient condition for S-irrelevance and S-independence, expressed in terms of triviality, precision and an interval version of the factorisation property. 
To state this factorisation condition, we adopt \(\smash{\lowuppprev(\altgbl)}\) as a shorthand notation for the interval \([\lowprev(\altgbl),\uppprev(\altgbl)]\) and employ the so-called \emph{interval product} \([a,b]\odot c\) of a real interval \([a,b]\) with a real number \(c\), defined as
\begin{equation*}
[a,b]\odot c\coloneqq[\min\set{ac,bc},\max\set{ac,bc}].
\end{equation*}
We also adopt the convention that a singleton is identified with its unique element, which allows us to write \(\lowuppprev(\altgbl)=\linprev(\altgbl)\) as a shorthand for \(\lowprev(\altgbl)=\uppprev(\altgbl)=\vcentcolon\linprev(\altgbl)\).

\begin{theorem}\label{theo:lowertwosided}
Consider any two events \(\aevt,\bevt\subseteq\pspace\) and any coherent lower prevision~\(\lowprev\) on~\(\gbls\).
Then \(\aevt\) is S-irrelevant to~\(\bevt\) with respect to~\(\rejectset[\lowprev]\) if and only if \(\aevt\) is trivial or
\begin{equation}
\label{eq:lower:twosided:irrelevance}
\lowuppprev(\altgbl)=\linprev(\altgbl)\text{ and }\lowuppprev(\gbl\altgbl)=\lowuppprev(\gbl)\odot\linprev(\altgbl)
\text{ for all \(\gbl\in\gbls[\aevt]\) and \(\altgbl\in\gbls[\bevt]\).}\\
\end{equation}
Similarly, \(\aevt\) and \(\bevt\) are S-independent if and only if \(\aevt\) or \(\bevt\) are trivial or
\begin{multline}\label{eq:lower:twosided:independence}
\text{\(\lowuppprev(\gbl)=\linprev(\gbl)\), \(\lowuppprev(\altgbl)=\linprev(\altgbl)\) and}\\
\lowuppprev(\gbl\altgbl)=\linprev(\gbl)\linprev(\altgbl)
\text{ for all \(\gbl\in\gbls[\aevt]\) and \(\altgbl\in\gbls[\bevt]\).}
\end{multline}
\end{theorem}

\begin{proof}
We begin with the first statement.
For necessity, assume that \(\aevt\) is non-trivial and S-irrelevant to~\(\bevt\). 
Consider any \(\gbl\in\gbls[\aevt]\) and \(\altgbl\in\gbls[\bevt]\).
Proposition~\ref{prop:lowerprecise} then guarantees that \(\lowprev(\altgbl)=\uppprev(\altgbl)\eqqcolon\linprev(\altgbl)\).
Consider now any \(\altlinprev\in\linprevs(\lowprev)\), then on the one hand \(\altlinprev(\altgbl)=\linprev(\altgbl)\) and on the other hand \(\altlinprev(\gbl\altgbl)=\altlinprev(\gbl)\altlinprev(\altgbl)=\altlinprev(\gbl)\linprev(\altgbl)\) by Proposition~\ref{prop:loweronesided}.
Hence by taking infima and suprema over all \(\altlinprev\in\linprevs(\lowprev)\) on both sides, we get that
\begin{equation*}
\lowprev(\gbl\altgbl)=
\begin{cases}
\lowprev(\gbl)\linprev(\altgbl)&\text{if \(\linprev(\altgbl)\geq0\)}\\
\uppprev(\gbl)\linprev(\altgbl)&\text{if \(\linprev(\altgbl)\leq0\)}
\end{cases}
\text{ and }
\uppprev(\gbl\altgbl)=
\begin{cases}
\uppprev(\gbl)\linprev(\altgbl)&\text{if \(\linprev(\altgbl)\geq0\)}\\
\lowprev(\gbl)\linprev(\altgbl)&\text{if \(\linprev(\altgbl)\leq0\)},
\end{cases}
\end{equation*}
which can indeed be summarised as \(\lowuppprev(\gbl\altgbl)=\lowuppprev(\gbl)\odot\linprev(\altgbl)\).

Next, we address sufficiency. 
Since we know from Proposition~\ref{prop:triviality:implies:S:independence} that the triviality of~\(\aevt\) implies its S-irrelevance to \(\bevt\), we can assume without loss of generality that Equation~\eqref{eq:lower:twosided:irrelevance} holds, and prove that \(\aevt\) is S-irrelevant to~\(\bevt\).
Consider any \(\gbl\in\gbls[\bevt]\). 
Equation~\eqref{eq:lower:twosided:irrelevance} then tells us that \(\smash{\lowuppprev(\gbl)=\linprev(\gbl)}\) and \(\smash{\lowuppprev(-\gbl)=\linprev(-\gbl)}\). 
Furthermore, since \(\lowprev(\gbl)=-\uppprev(-\gbl)\), we find that \(\linprev(\gbl)=-\linprev(-\gbl)\). 
We now consider two cases: \(\linprev(\gbl)\geq0\) and \(\linprev(\gbl)\leq0\). 
If \(\linprev(\gbl)\geq0\), we infer from Equation~\eqref{eq:lower:twosided:irrelevance} that \(\smash{\lowprev(\ind{\co{\aevt}}\gbl)=\lowprev(\co{\aevt})\linprev(\gbl)\geq0}\). 
If \(\linprev(\gbl)\leq0\), then \(\linprev(-\gbl)=-\linprev(\gbl)\geq0\), so it follows from Equation~\eqref{eq:lower:twosided:irrelevance} that \(\smash{\lowprev(-\ind{\aevt}\gbl)=\lowprev(\ind{\aevt}(-\gbl))=\lowprev(\aevt)\linprev(-\gbl)\geq0}\). 
We conclude that in all cases, \(\lowprev(\ind{\co{\aevt}}\gbl)\geq0\) or \(\lowprev(-\ind{\aevt}\gbl)\geq0\). 
Since this is true for every \(\gbl\in\gbls[\bevt]\), it follows that, indeed, \(\aevt\) is S-irrelevant to~\(\bevt\) with respect to~\(\rejectset[\lowprev]\).

Next, we turn to the second statement.
For necessity, assume that \(\aevt\) and \(\bevt\) are non-trivial and S-independent.
Proposition~\ref{prop:lowerprecise} then guarantees that \(\smash{\lowprev(\gbl)=\uppprev(\gbl)\eqqcolon\linprev(\gbl)}\) and \(\smash{\lowprev(\altgbl)=\uppprev(\altgbl)\eqqcolon\linprev(\altgbl)}\) for all \(\gbl\in\gbls[\aevt]\) and \(\altgbl\in\gbls[\bevt]\).
For all \(\altlinprev\in\linprevs(\lowprev)\), Proposition~\ref{prop:loweronesided} then guarantees that \(\altlinprev(\gbl\altgbl)=\altlinprev(\gbl)\altlinprev(\altgbl)=\linprev(\gbl)\linprev(\altgbl)\), and therefore, by taking infima and suprema over all \(\altlinprev\in\linprevs(\lowprev)\) on both sides, we get that, indeed, \(\smash{\lowprev(\gbl\altgbl)=\uppprev(\gbl\altgbl)=\linprev(\gbl)\linprev(\altgbl)}\).

We now turn to sufficiency. 
Since we know from Proposition~\ref{prop:triviality:implies:S:independence} that the triviality of~\(\aevt\) or \(\bevt\) implies their S-independence, we can assume without loss of generality that Equation~\eqref{eq:lower:twosided:independence} holds, and prove that \(\aevt\) and \(\bevt\) are S-independent. 
Since Equation~\eqref{eq:lower:twosided:independence} implies Equation~\eqref{eq:lower:twosided:irrelevance}, the S-irrelevance of~\(\aevt\) to \(\bevt\) follows from the first part of this theorem. 
Since Equation~\eqref{eq:lower:twosided:independence} is symmetric in~\(\aevt\) and \(\bevt\), the S-irrelevance of~\(\bevt\) to \(\aevt\) follows in exactly the same way. 
Hence, we find that, indeed, \(\aevt\) and \(\bevt\) are S-independent with respect to~\(\rejectset[\lowprev]\).
\end{proof}

In combination with Proposition~\ref{prop:S-irrelevance:in:terms_of:binary}, this Theorem~\ref{theo:lowertwosided} also yields characterisations for S-irrelevance and S-independence for Archimedean sets of desirable option sets~\(\rejectset\), in terms of their representing lower previsions. 
In particular, we see that in the absence of triviality, S-independence implies precision and factorisation for (products of) gambles on~\(\aevt\) and \(\bevt\), \emph{without the need for imposing mixingness}. 
In the remainder of this section, we seek to exclude the trivial cases by considering a new potential property for events with respect to a coherent set of desirable option sets~\(\rejectset\).

We start by introducing the notion of credibility: we say that \(\aevt\) is \emph{credible} with respect to a coherent set of desirable option sets~\(\rejectset\) whenever 
\begin{equation*}
(\exists\epsilon>0)\set{\ind{\aevt}-\epsilon}\in\rejectset,
\end{equation*}
meaning that our subject is willing to bet on~\(\aevt\) at some positive---but possibly very small---betting rate~\(\epsilon\). 
For an Archimedean \(\rejectset\), this is equivalent to the lower probability of~\(\aevt\) being strictly bounded below by~\(\epsilon\) for all representing lower previsions.

\begin{proposition}\label{prop:credible:in:terms:of:lowprev}
If the set of desirable option sets~\(\rejectset\) is Archimedean, an event \(\aevt\subseteq\pspace\) is credible with respect to~\(\rejectset\) if and only if there is some real \(\epsilon>0\) such that \(\lowprev(\aevt)>\epsilon\) for all \(\lowprev\in\cohlowprevs(\rejectset)\).
\end{proposition}

\begin{proof}
Consider any \(\epsilon>0\). 
Then for any \(\lowprev\in\cohlowprevs(\rejectset)\), \(\set{\ind{\aevt}-\epsilon}\in\rejectset[\lowprev]\) if and only if \(\lowprev(\ind{\aevt}-\epsilon)>0\), or equivalently---since lower previsions are constant additive [\ref{ax:lowprev:constantadditive}]---\(\lowprev(\ind{\aevt})>\epsilon\). 
The result is therefore an immediate consequence of the fact that \(\rejectset=\bigcap\cset{\rejectset[\lowprev]}{\lowprev\in\cohlowprevs(\rejectset)}\).
\end{proof}

We now say that an event~\(\aevt\) is \emph{credibly indeterminate} with respect to a coherent set of desirable option sets~\(\rejectset\) if \(\aevt\) and \(\co{\aevt}\) are both credible with respect to~\(\rejectset\), meaning that our subject is willing to bet both on and against \(\aevt\) at some positive betting rate. 
This condition of credible indeterminacy, when combined with S-independence, allows us to infer both precision and factorisation for every representing lower prevision of an Archimedean set of desirable option sets~\(\rejectset\), even without our having to impose mixingness!

\begin{theorem}\label{theo:loweronesided:K}
Consider any two events \(\aevt,\bevt\subseteq\pspace\) and an Archimedean set of desirable option sets~\(\rejectset\). 
If \(\aevt\) is credibly indeterminate and S-irrelevant to~\(\bevt\) with respect to~\(\rejectset\), then for all \(\lowprev\in\cohlowprevs(\rejectset)\):
\begin{equation*}
\lowuppprev(\altgbl)=\linprev(\altgbl)\text{ and }\lowuppprev(\gbl\altgbl)=\lowuppprev(\gbl)\odot\linprev(\altgbl)
\text{ for all \(\gbl\in\gbls[\aevt]\) and \(\altgbl\in\gbls[\bevt]\).}\\
\end{equation*}
Similarly, if \(\aevt\) and \(\bevt\) are credibly indeterminate and S-independent with respect to~\(\rejectset\), then for all \(\lowprev\in\cohlowprevs(\rejectset)\):
\begin{equation*}
\lowuppprev(\gbl)=\linprev(\gbl),\lowuppprev(\altgbl)=\linprev(\altgbl)\text{ and }
\lowuppprev(\gbl\altgbl)=\linprev(\gbl)\linprev(\altgbl)
\text{ for all \(\gbl\in\gbls[\aevt]\) and \(\altgbl\in\gbls[\bevt]\).}
\end{equation*}
\end{theorem}

\begin{proof}
Due to Proposition~\ref{prop:S-irrelevance:in:terms_of:binary} and Theorem~\ref{theo:lowertwosided}, it clearly suffices to show that the credible indeterminacy of an event \(\aevt\) implies that \(\aevt\) is non-trivial with respect to every \(\lowprev\in\cohlowprevs(\rejectset)\). 
So assume that \(\aevt\) is credibly indeterminate with respect to~\(\rejectset\), meaning that \(\taevt\) is credible for each \(\taevt\in\set{\aevt,\co{\aevt}}\). 
For any \(\lowprev\in\cohlowprevs(\rejectset)\), it then follows from Proposition~\ref{prop:credible:in:terms:of:lowprev} that there is some \(\epsilon>0\) such that \(\lowprev(\taevt)>\epsilon\). 
Hence, since \(\smash{\uppprev(\taevt)\geq\lowprev(\taevt)>\epsilon>0}\), we see that \(\aevt\) is indeed non-trivial with respect to~\(\lowprev\).
\end{proof}

\section{S-Irrelevance for variables}\label{sec:S-irrelevance:for:variables}
Let us now extend the discussion from events to variables.
We still assume that our subject's preferences are modelled by a set of desirable option sets~\(\rejectset\), where the possible options are the gambles on the possibility space~\(\pspace\).

We will follow the often used device of representing a variable~\(\zvar\) as a map that is defined on the possibility space~\(\pspace\):\footnote{Although it is related to what we call a variable in spirit, we will refrain from using the term `random variable', as that is typically associated with precise and countable additive probability models, and typically comes with a measurability requirement.}
\begin{equation*}
\zvar\colon\pspace\to\zvalues\colon\omega\mapsto\zvar(\omega),
\end{equation*}
where we denote by~\(\zvalues\) the set of possible values of the variable \(\zvar\).
The idea behind this device is that since our subject is uncertain about the value that \(\omega\) assumes in~\(\pspace\), she will typically also be uncertain about the value assumed by~\(\zvar(\omega)\) in~\(\zvalues\).

If we want to talk about decisions involving the value of the variable \(\zvar\), we need to consider uncertain rewards whose value depends only on the value of~\(\zvar\), or more specifically, gambles on~\(\pspace\) of the type
\begin{equation*}
\altgbltoo(\zvar)\coloneqq\altgbltoo\circ\zvar\colon\pspace\to\reals\colon\omega\mapsto\altgbltoo(\zvar(\omega)),
\end{equation*}
where \(\altgbltoo\) is any gamble on~\(\zvalues\).

In particular, with any \(\aevt\subseteq\zvalues\), we can associate the indicator (gamble) \(\ind{\aevt}\) on~\(\zvalues\), which corresponds to a gamble \(\ind{\aevt}(\zvar)\) on~\(\pspace\).
Since \(\ind{\aevt}(\zvar)=\ind{\aevt}\circ\zvar=\ind{\zvar^{-1}(\aevt)}\), we see that \(\ind{\aevt}(\zvar)\) is an indicator on \(\pspace\). The event \(\smash{\zvar^{-1}(\aevt)\subseteq\pspace}\) that it indicates, corresponds to the proposition `\(\zvar\in\aevt\)'.

\subsection{Defining S-irrelevance for variables}
To see what an assessment of S-irrelevance might mean for variables, we consider two variables~\(\xvar\) and~\(\yvar\), which we model as maps on the possibility space~\(\pspace\):
\begin{equation*}
\xvar\colon\pspace\to\xvalues\text{ and }\yvar\colon\pspace\to\yvalues,
\end{equation*} 
where the respective non-empty sets~\(\xvalues\) and~\(\yvalues\) are the sets of possible values for~\(\xvar\) and~\(\yvar\).
Seidenfeld's `independence' requirement can then be extended straightforwardly from events to variables as follows:
\begin{quote}
When two variables, \(\xvar\) and \(\yvar\), are `independent' then it is not reasonable to \emph{spend resources} in order to use the observed value of one of them, say \(\xvar\), to choose between options that depend solely on the value of the other variable, \(\yvar\).
\end{quote}
As before for events, we recognise the essentially asymmetrical nature of this requirement, and will try to formulate a requirement of S-irrelevance of~\(\xvar\) to \(\yvar\).
Because we care about the operational meaning of our criterion, we will allow the variables~\(\xvar\) and~\(\yvar\) to assume infinitely many values, but want to keep our observations of the values of~\(\xvar\) and the choices between gambles on~\(\yvar\) finitary.
Consequently, observing the value of~\(\xvar\) will be modelled by choosing a \emph{finite} partition~\(\partition\) of the set \(\xvalues\), and finding out which event \(\aevt\in\partition\) in that partition obtains, meaning that \(\xvar\in\aevt\).
For each such possible observation~\(\aevt\), we consider a gamble \(s_{\aevt}\colon\yvalues\to\reals\) on the value of~\(\yvar\), which we will assume to be \emph{finite-valued}---simple---in accordance with our finitary approach.\footnote{An engaged reader will be able to verify further on that, despite our insistence on a finitary approach, this doesn't really matter from a mathematical point of view. In particular, none of our proofs---even the ones that establish \emph{sufficient} conditions for S-irrelevance or S-independence for variables---will actually require the restriction that the gambles \(s_E\)---or \(s_G\)---should be simple.}
Let us denote by~\(\sgbls[\yvalues]\) the set of all simple gambles on~\(\yvalues\).

In summary, our subject needs to choose between the gambles \(s_{\aevt}\), \(\aevt\in\partition\) on the value of~\(\yvar\), or in other words, between the gambles \(s_{\aevt}(\yvar)\coloneqq s_{\aevt}\circ \yvar\) on the possibility space~\(\pspace\).
Observing the value of~\(\xvar\) to choose between these gambles then corresponds to the composite gamble \(\sum_{\cevt\in\partition}\ind{\cevt}(\xvar)s_{\cevt}(\yvar)\) on the possibility space~\(\pspace\).
The purport of Seidenfeld's requirement is that when our subject judges \(\xvar\) to be irrelevant to~\(\yvar\), she should reject the composite option~\(\sum_{\cevt\in\partition}\ind{\cevt}(\xvar)s_{\cevt}(\yvar)-\epsilon\) from the set of options
\begin{equation*}
\cset{s_{\aevt}(\yvar)}{\aevt\in\partition}\cup\set[\bigg]{\sum_{\cevt\in\partition}\ind{\cevt}(\xvar)s_{\cevt}(\yvar)-\epsilon},
\end{equation*}
for all real \(\epsilon>0\).
Following the discussion in Sections~\ref{sec:choice-functions} and~\ref{sec:S-irrelevance:for:events}, and Equation~\eqref{eq:choiceandrejectionfromK} in particular, this means that the option set
\begin{multline*}
\cset[\bigg]{s_{\aevt}(\yvar)-\sum_{\cevt\in\partition}\ind{\cevt}(\xvar)s_{\cevt}(\yvar)+\epsilon}{\aevt\in\partition}\\
\begin{aligned}
&=\cset[\bigg]{\sum_{\cevt\in\partition}\ind{\cevt}(\xvar)s_{\aevt}(\yvar)-\sum_{\cevt\in\partition}\ind{\cevt}(\xvar)s_{\cevt}(\yvar)+\epsilon}{\aevt\in\partition}\\
&=\cset[\bigg]{\sum_{\cevt\in\partition\setminus\set{\aevt}}\ind{\cevt}(\xvar)[s_{\aevt}(\yvar)-s_{\cevt}(\yvar)]+\epsilon}{\aevt\in\partition}
\end{aligned}
\end{multline*}
must be desirable for our subject.
This leads to the following definitions.

\begin{definition}[S-irrelevance and S-independence for variables]\label{def:S-irrelevance:for:variables}
Consider two variables~\(\xvar\) and~\(\yvar\) and a coherent set of desirable option sets~\(\rejectset\).
We say that \(\xvar\) is \emph{S-irrelevant} to \(\yvar\) with respect to~\(\rejectset\) if
\begin{multline}\label{eq:simplifiedphysirr:for:variables}
\cset[\bigg]{\sum_{\cevt\in\partition\setminus\set{\aevt}}\ind{\cevt}(\xvar)[s_{\aevt}(\yvar)-s_{\cevt}(\yvar)]+\epsilon}{\aevt\in\partition}\in\rejectset\\
\text{ for all finite partitions \(\partition\) of~\(\xvalues\), all \(s_{\aevt}\in\sgbls[\yvalues]\) and all \(\epsilon>0\)}.
\end{multline}
We say that \(\xvar\) and \(\yvar\) are \emph{S-independent} with respect to~\(\rejectset\) if \(\xvar\) is S-irrelevant to~\(\yvar\) and \(\yvar\) is S-irrelevant to~\(\xvar\).
\end{definition}

Interestingly, the assessments for variables imply similar assessments for events.

\begin{proposition}\label{prop:variableimpliesevents}
Consider two variables~\(\xvar\) and~\(\yvar\) and a coherent set of desirable option sets~\(\rejectset\).
If \(\xvar\) is S-irrelevant to~\(\yvar\) with respect to~\(\rejectset\), then for all \(\aevt\subseteq\xvalues\) and \(\bevt\subseteq\yvalues\), the event~\(\xvar^{-1}(\aevt)\) is S-irrelevant to the event~\(\yvar^{-1}(\bevt)\) with respect to~\(\rejectset\). 
Similarly, if \(\xvar\) and \(\yvar\) are S-independent with respect to~\(\rejectset\), then for all \(\aevt\subseteq\xvalues\) and \(\bevt\subseteq\yvalues\), the events \(\xvar^{-1}(\aevt)\) and \(\yvar^{-1}(\bevt)\) are S-independent with respect to~\(\rejectset\). 
\end{proposition}

\begin{proof}
It clearly suffices to give the proof for S-irrelevance.
So consider any \(\aevt\subseteq\xvalues\) and \(\bevt\subseteq\yvalues\) and any gamble \(\smash{\gbl\in\gbls[\yvar^{-1}(\bevt)](\pspace)}\).
Then \(\gbl\) is completely characterised by the real values \(\gbli[\bevt]\) and \(\gbli[\co{\bevt}]\) it assumes on \(\yvar^{-1}(\bevt)\) and \(\co{(\yvar^{-1}(\bevt))}=\yvar^{-1}(\co{\bevt})\), respectively: \(\gbl=\gbli[\bevt]\ind{\yvar^{-1}(\bevt)}+\gbli[\co{\bevt}]\ind{\yvar^{-1}(\co{\bevt})}\).
If we define the simple gamble \(\altgbl\coloneqq\gbli[\bevt]\ind{\bevt}+\gbli[\co{\bevt}]\ind{\co{\bevt}}\in\sgbls[\yvalues]\), then clearly \(\gbl=\altgbl(\yvar)\). 
We now consider the partition \(\partition\coloneqq\set{\aevt,\co{\aevt}}\) of \(\xvalues\) and the corresponding simple gambles \(s_{\aevt}\coloneqq\altgbl\) and \(s_{\co{\aevt}}\coloneqq0\) on \(\yvalues\).  
Then the S-irrelevance of \(\xvar\) to \(\yvar\) implies in particular that for all \(\epsilon>0\),
\begin{multline*}
\rejectset\ni\set[\big]{\ind{\co{\aevt}}(\xvar)[s_{\aevt}(\yvar)-s_{\co{\aevt}}(\yvar)]+\epsilon,
\ind{\aevt}(\xvar)[s_{\co{\aevt}}(\yvar)-s_{\aevt}(\yvar)]+\epsilon}\\
=\set[\big]{\ind{\co{\aevt}}(\xvar)g(\yvar)+\epsilon,-
\ind{\aevt}(\xvar)g(\yvar)+\epsilon}
=\set[\big]{\ind{\xvar^{-1}(\co{\aevt})}\gbl+\epsilon,
-\ind{\xvar^{-1}(\aevt)}\gbl+\epsilon}.
\end{multline*}
Since \(\xvar^{-1}(\co{\aevt})=\co{(\xvar^{-1}(\aevt))}\), this indeed tells us that \(\xvar^{-1}(\aevt)\) is S-irrelevant to~\(\yvar^{-1}(\bevt)\) with respect to~\(\rejectset\).
\end{proof}

A straightforward generalisation, {\itshape mutatis mutandis}, of Proposition~\ref{prop:S-irrelevance:in:terms_of:binary:variables} indicates that S-irrelevance and S-independence of variables with respect to Archimedean (and mixing) models are completely determined by the corresponding notions for their dominating binary models.

\begin{proposition}\label{prop:S-irrelevance:in:terms_of:binary:variables}
Consider two variables \(\xvar\) and \(\yvar\) and a set of desirable option sets~\(\rejectset\).
If \(\rejectset\) is Archimedean, then \(\xvar\) is S-irrelevant to~\(\yvar\) with respect to~\(\rejectset\) if and only if \(\xvar\) is S-irrelevant to~\(\yvar\) with respect to~\(\rejectset[\lowprev]\) for all \(\lowprev\in\cohlowprevs(\rejectset)\).
Similarly, if \(\rejectset\) is Archimedean and mixing, then \(\xvar\) is S-irrelevant to~\(\yvar\) with respect to~\(\rejectset\) if and only if \(\xvar\) is S-irrelevant to~\(\yvar\) with respect to~\(\rejectset[\linprev]\) for all \(\linprev\in\linprevs(\rejectset)\).
\end{proposition}

\begin{proof}
If we combine Definitions~\ref{def:archimedeanity} and~\ref{def:S-irrelevance:for:variables}, we see that \(\xvar\) is S-irrelevant to~\(\yvar\) with respect to an Archimedean~\(\rejectset\) if and only if
\begin{multline*}
\cset[\bigg]{\sum_{\cevt\in\partition\setminus\set{\aevt}}\ind{\cevt}(\xvar)[s_{\aevt}(\yvar)-s_{\cevt}(\yvar)]+\epsilon}{\aevt\in\partition}\in\bigcap\cset{\rejectset[\lowprev]}{\lowprev\in\cohlowprevs(\rejectset)}\\
\text{ for all finite partitions \(\partition\) of~\(\xvalues\), all \(s_{\aevt}\in\sgbls[\yvalues]\) and all \(\epsilon>0\)}.
\end{multline*}
which proves the statement for Archimedeanity.
If we also impose mixingness on~\(\rejectset\), the second statement follows at once from the first and the fact that then, according to Theorem~\ref{thm:archimedean:and:mixing}, \(\cohlowprevs(\rejectset)=\linprevs(\rejectset)\).
\end{proof}

\subsection{S-irrelevance for variables with respect to lower prevision models}\label{sec:irrelevance:variables:lower:previsions}
Because of Proposition~\ref{prop:S-irrelevance:in:terms_of:binary:variables}, it will be useful in this section to consider and study S-irrelevance with respect to the binary model~\(\rejectset[\lowprev]\) associated with a coherent lower prevision~\(\lowprev\) on~\(\gbls\). 
Clearly, \(\xvar\) will be S-irrelevant to~\(\yvar\) with respect to~\(\rejectset[\lowprev]\) if and only if
\begin{multline}\label{eq:S-irrelevance:lowprev:for:variables}
\max\cset[\bigg]{\lowprev\group[\bigg]{\sum_{\cevt\in\partition\setminus\set{\aevt}}\ind{\cevt}(\xvar)[s_{\aevt}(\yvar)-s_{\cevt}(\yvar)]}}{\aevt\in\partition}\geq0\\
\text{ for all finite partitions \(\partition\) of~\(\xvalues\) and all \(s_{\aevt}\in\sgbls[\yvalues]\)}.
\end{multline}

As a first step, we recall from Proposition~\ref{prop:variableimpliesevents} that S-irrelevance for variables implies S-irrelevance for events.
This allows us to apply some of the results in Section~\ref{sec:S-irrelevance:for:events} and generalise them to the present context of variables. 
To get our feet wet, we look at Proposition~\ref{prop:loweronesided}, whose generalisation requires a notion of independent variables for linear previsions.

\begin{definition}[Independent variables with respect to a linear prevision]\label{def:independenceforvariables}
We call two variables \(\xvar\) and \(\yvar\) \emph{independent} with respect to a linear prevision~\(\linprev\) on \(\gbls\) if for all \(\aevt\subseteq\xvalues\) and \(\bevt\subseteq\yvalues\), the events \(\xvar^{-1}(\aevt)\) and \(\yvar^{-1}(\bevt)\) are independent with respect to~\(\linprev\).
\end{definition}
\noindent Its characterisation in terms of gambles on~\(\xvar\) and~\(\yvar\) follows a fairly standard argument based on the uniform density of the simple gambles for the set of all gambles. 
We include the simple proof for the sake of completeness.

\begin{proposition}\label{prop:independenceforvariables}
For any linear prevision \(\linprev\), two variables \(\xvar\) and \(\yvar\) are independent with respect to~\(\linprev\) if and only if
\begin{equation*}
\linprev(\gbl\/(\xvar)\altgbl(\yvar))=\linprev(\gbl(\xvar))\linprev(\altgbl(\yvar))
\text{ for all gambles \(\gbl\) on~\(\xvalues\) and all gambles \(\altgbl\) on~\(\yvalues\)}.
\end{equation*}
\end{proposition}

\begin{proof}
The sufficiency part of the proof is immediate: it suffices to let \(f\coloneqq\ind{\aevt}\) and \(g\coloneqq\ind{\bevt}\).
For the necessity part, we assume that \(\xvar\) and \(\yvar\) are independent with respect to~\(\linprev\) and consider any gamble \(\gbl\) on \(\xvalues\) and \(\altgbl\) on \(\yvalues\). 
We need to prove that \(\linprev(\gbl(\xvar)\altgbl(\yvar))=\linprev(\gbl(\xvar))\linprev(\altgbl(\yvar))\). 
It suffices to give the proof for simple---finite valued---gambles~\(\gbl\) and~\(\altgbl\), because any gamble is a uniform limit of simple gambles, and because coherent lower previsions, and hence also linear previsions, are guaranteed to be uniformly continuous [see~\ref{ax:lowprev:uniformconvergence}].
We may therefore assume that \(\gbl=\sum_{\aevt\in\partition[\xvalues]}\gbli[\aevt]\ind{\aevt}\) for some partition~\(\partition[\xvalues]\) of~\(\xvalues\) and some choice of the real numbers~\(\gbli[\aevt]\).
Similarly, we may assume that \(\altgbl=\sum_{\bevt\in\partition[\yvalues]}\altgbli[\bevt]\ind{\bevt}\) for some partition~\(\partition[\xvalues]\) of~\(\yvalues\) and some choice of the real numbers~\(\altgbli[\bevt]\).
But then
\begin{align*}
\linprev(\gbl(\xvar)\altgbl(\yvar))
&=\linprev\group[\bigg]{\sum_{\aevt\in\partition[\xvalues]}\sum_{\bevt\in\partition[\yvalues]}\gbli[\aevt]\altgbli[\bevt]\ind{\aevt}(\xvar)\ind{\bevt}(\yvar)}\\
&=\sum_{\aevt\in\partition[\xvalues]}\sum_{\bevt\in\partition[\yvalues]}\gbli[\aevt]\altgbli[\bevt]\linprev(\ind{\aevt}(\xvar)\ind{\bevt}(\yvar))\\
&=\sum_{\aevt\in\partition[\xvalues]}\sum_{\bevt\in\partition[\yvalues]}\gbli[\aevt]\altgbli[\bevt]\linprev(\ind{\aevt}(\xvar))\linprev(\ind{\bevt}(\yvar))\\
&=\sum_{\aevt\in\partition[\xvalues]}\gbli[\aevt]\linprev(\ind{\aevt}(\xvar))\sum_{\bevt\in\partition[\yvalues]}\altgbli[\bevt]\linprev(\ind{\bevt}(\yvar))\\
&=\linprev\group[\bigg]{\sum_{\aevt\in\partition[\xvalues]}\gbli[\aevt]\ind{\aevt}(\xvar)}
\linprev\group[\bigg]{\sum_{\bevt\in\partition[\yvalues]}\altgbli[\bevt]\ind{\bevt}(\yvar)}
=\linprev(\gbl(\xvar))\linprev(\altgbl(\yvar)).
\end{align*}
where the second and fifth equality follow from coherence [use~\ref{ax:linprev:homo} and~\ref{ax:linprev:additive}], and the crucial third equality follows from the assumption that \(\xvar\) and \(\yvar\) are independent with respect to~\(\linprev\).
\end{proof}

With independence out of the way, we can now address, as announced, the generalisation of Proposition~\ref{prop:loweronesided} from events to variables.

\begin{proposition}\label{prop:loweronesided:variables}
Consider any coherent lower prevision~\(\lowprev\) on \(\gbls\), and two variables \(\xvar\) and \(\yvar\). 
If for all \(\aevt\subseteq\xvalues\) and \(\bevt\subseteq\yvalues\), \(\xvar^{-1}(\aevt)\) is S-irrelevant to~\(\yvar^{-1}(\bevt)\) with respect to~\(\rejectset[\lowprev]\), then \(\xvar\) and \(\yvar\) are independent with respect to all \(\linprev\in\linprevs(\lowprev)\).
\end{proposition}

\begin{proof}
Consider any \(\linprev\in\linprevs(\lowprev)\) and any \(\aevt\subseteq\xvalues\) and \(\bevt\subseteq\yvalues\). 
Then since \(\xvar^{-1}(\aevt)\) is S-irrelevant to~\(\yvar^{-1}(\bevt)\) with respect to~\(\rejectset[\lowprev]\), it follows from Proposition~\ref{prop:loweronesided} that \(\xvar^{-1}(\aevt)\) and \(\yvar^{-1}(\bevt)\) are independent with respect to~\(\linprev\).
Now use Proposition~\ref{prop:independenceforvariables}.
\end{proof}

To extend Proposition~\ref{prop:lowerprecise} to the context of variables, we first need to introduce concepts of triviality and precision for variables.

We begin with precision.
It makes perfect sense to call the lower prevision~\(\lowdis[\zvar]\) on the set \(\gbls(\zvalues)\) of all gambles on~\(\zvalues\), defined by
\begin{equation*}
\lowdis[\zvar](\altgbltoo)\coloneqq\lowprev(\altgbltoo\circ\zvar)\text{ for all gambles \(\altgbltoo\) on~\(\zvalues\)}
\end{equation*}
the (lower) \emph{distribution} of the variable \(\zvar\) with respect to the coherent lower prevision~\(\lowprev\).
We will say that a variable \(\zvar\colon\pspace\to\zvalues\) has a \emph{precise distribution} with respect to~\(\lowprev\) if the gamble \(\altgbltoo(\zvar)=\altgbltoo\circ\zvar\) on~\(\pspace\) has a precise prevision~\(\lowprev(\altgbltoo\circ\zvar)=\uppprev(\altgbltoo\circ\zvar)\eqqcolon\dis[\zvar](\altgbltoo)\) for all gambles \(\altgbltoo\) on~\(\zvalues\), or in other words if the distribution~\(\lowdis[\zvar]\) of~\(\zvar\) with respect to~\(\lowprev\) is a linear prevision, then denoted by~\(\dis[\zvar]\).

This notion of precision can also be expressed in terms of events.

\begin{proposition}\label{prop:precise:distribution}
A variable \(\zvar\colon\pspace\to\zvalues\) has a precise distribution with respect to a coherent lower prevision~\(\lowprev\) if and only if the events \(\zvar^{-1}(\aevt)\) have a precise probability with respect to~\(\lowprev\) for all \(\aevt\subseteq\zvalues\).
\end{proposition}

\begin{proof}
Necessity is immediate, so we concentrate on sufficiency.
We assume that the events \(\zvar^{-1}(\aevt)\) have a precise probability for all \(\aevt\subseteq\zvalues\), and prove that the gambles \(\altgbltoo(\zvar)=\altgbltoo\circ\zvar\) have a precise prevision for all gambles \(\altgbltoo\) on~\(\zvalues\).
Since a coherent lower prevision is uniformly continuous [see~\ref{ax:lowprev:uniformconvergence}], and since all gambles are uniform limits of simple gambles, it suffices to give the proof for simple gambles \(s=\sum_{k=1}^ns_k\ind{\aevt[k]}\), where the \(\aevt[k]\) constitute a partition of~\(\zvalues\) and the \(s_k\in\reals\).
We may assume without loss of generality that \(s\) is non-negative, so all \(s_k\geq0\), due to the constant additivity [\ref{ax:lowprev:constantadditive}] of a coherent lower prevision.
Hence, indeed,
\begin{align*}
\lowprev(s\circ\zvar)
&=\lowprev\group[\bigg]{\sum_{k=1}^ns_k\ind{\zvar^{-1}(\aevt[k])}}
\geq\sum_{k=1}^ns_k\lowprev(\zvar^{-1}(\aevt[k]))\\
&=\sum_{k=1}^ns_k\uppprev(\zvar^{-1}(\aevt[k]))
\geq\uppprev\group[\bigg]{\sum_{k=1}^ns_k\ind{\zvar^{-1}(\aevt[k])}}
=\uppprev(s\circ\zvar),
\end{align*}
where the first inequality follows from the super-linearity of the coherent lower prevision~\(\lowprev\) [combine~\ref{ax:lowprev:superadditive} and~\ref{ax:lowprev:homo}], the second equality from the assumption, and the second inequality from the sub-linearity of the coherent upper prevision~\(\smash{\uppprev}\) [combine~\ref{ax:lowprev:mixed:additivity} and~\ref{ax:lowprev:homo}].
\end{proof}

Let us call a variable \(\zvar\colon\pspace\to\zvalues\) \emph{trivial} with respect to a coherent lower prevision~\(\lowprev\) if \(\zvar^{-1}(\cevt)\) is trivial with respect to~\(\lowprev\) for all subsets \(\cevt\subseteq\zvalues\), meaning that \(\smash{\uppdis[\zvar](\cevt)=0}\) or \(\smash{\uppdis[\zvar](\co{\cevt})=0}\) for all \(\cevt\subseteq\zvalues\).
The distribution~\(\lowdis[\zvar]\) is then clearly precise on all events, and therefore also a linear prevision~\(\dis[\zvar]\) on all gambles [use Proposition~\ref{prop:precise:distribution}].
It is the---degenerate---linear prevision given by
\begin{equation*}
\dis[\zvar](\altgbltoo)
=\sup_{\aevt\in\ultrafilter{\zvar}}\inf_{z\in\aevt}\altgbltoo(z)
=\inf_{\aevt\in\ultrafilter{\zvar}}\sup_{z\in\aevt}\altgbltoo(z)
\text{ for all gambles \(\altgbltoo\) on~\(\zvalues\)},
\end{equation*}
where the collection of \emph{practically certain events}
\begin{equation*}
\ultrafilter{\zvar}
\coloneqq\cset{\cevt\subseteq\zvalues}{\lowprev(\zvar^{-1}(\cevt))=1}
=\cset{\cevt\subseteq\zvalues}{\lowdis[\zvar](\cevt)=1}
\end{equation*}
is an ultrafilter of events on~\(\zvalues\); see for instance \cite[Section~5.5]{troffaes2013:lp} and \cite[Sections~2.9.8 and~3.2.6]{walley1991}.
If the ultrafilter \(\ultrafilter{\zvar}\) is \emph{fixed}, meaning that \(\bigcap\ultrafilter{\zvar}=\set{z_o}\) for some \(z_o\in\zvalues\), then \(\dis[\zvar](\altgbltoo)=\altgbltoo(z_o)\), so all probability mass of the precise distribution~\(\dis[\zvar]\) is concentrated in~\(z_o\).
The only other possibility is that the ultrafilter \(\ultrafilter{\zvar}\) is \emph{free}, meaning that \(\bigcap\ultrafilter{\zvar}=\emptyset\), and then typically all probability mass will lie infinitesimally close to some \(z_o\) in~\(\zvalues\), or to some `point on the boundary' of~\(\zvalues\).
In both cases, this represents a model for our subject's certainty that \(\zvar\) assumes a fixed value; see also the extensive discussion in Section~5.5.5 of \cite{troffaes2013:lp}.

Proposition~\ref{prop:lowerprecise} now generalises fairly easily from events to variables.

\begin{proposition}\label{prop:lowerprecise:variables}
Consider a coherent lower prevision~\(\lowprev\) and two variables \(\xvar\) and \(\yvar\). 
If for all \(\aevt\subseteq\xvalues\) and \(\bevt\subseteq\yvalues\), \(\xvar^{-1}(\aevt)\) is S-irrelevant to~\(\yvar^{-1}(\bevt)\) with respect to~\(\rejectset[\lowprev]\), then \(\xvar\) is trivial with respect to~\(\lowprev\) or \(\yvar\) has a precise distribution with respect to~\(\lowprev\).
\end{proposition}

\begin{proof}
Assume that \(\xvar\) is not trivial with respect to~\(\lowprev\). 
We then need to show that \(\yvar\) has a precise distribution with respect to~\(\lowprev\). 
Due to Proposition~\ref{prop:precise:distribution}, it suffices to consider any \(\bevt\subseteq\yvalues\) and prove that \(\smash{\lowprev(\yvar^{-1}(\bevt))=\uppprev(\yvar^{-1}(\bevt))}\).

Since \(\xvar\) is not trivial with respect to~\(\lowprev\), there is some \(\aevt\subseteq\xvalues\) such that the event \(\xvar^{-1}(\aevt)\) is not trivial with respect to~\(\lowprev\). 
Since \(\xvar\) is S-irrelevant to~\(\yvar\) with respect to~\(\rejectset[\lowprev]\), we also know that \(\xvar^{-1}(\aevt)\) is S-irrelevant to~\(\yvar^{-1}(\bevt)\) with respect to~\(\rejectset[\lowprev]\), by Proposition~\ref{prop:variableimpliesevents}. 
It therefore follows from Proposition~\ref{prop:lowerprecise} that \(\smash{\lowprev(\yvar^{-1}(\bevt))=\uppprev(\yvar^{-1}(\bevt))}\), as required.
\end{proof}

We showed in Proposition~\ref{prop:variableimpliesevents} that S-irrelevance for variables implies S-irrelevance for the corresponding families of events. 
It turns out that for sets of desirable option sets that are Archimedean, these notions are equivalent; see Theorem~\ref{theo:lower:equivalentdefinitions:K} further on. 
We start out by establishing this result for binary sets of desirable option sets of the form~\(\rejectset[\lowprev]\), using the results in Propositions~\ref{prop:variableimpliesevents},~\ref{prop:independenceforvariables},~\ref{prop:loweronesided:variables} and~\ref{prop:lowerprecise:variables}.

\begin{theorem}\label{theo:lower:equivalentdefinitions}
Consider a coherent lower prevision~\(\lowprev\) on \(\gbls\), and two variables~\(\xvar\) and~\(\yvar\). 
Then \(\xvar\) is S-irrelevant to~\(\yvar\) with respect to~\(\rejectset[\lowprev]\) if and only if, for all \(\aevt\subseteq\xvalues\) and \(\bevt\subseteq\yvalues\), \(\smash{\xvar^{-1}(\aevt)}\) is S-irrelevant to~\(\smash{\yvar^{-1}(\bevt)}\) with respect to~\(\rejectset[\lowprev]\).
\end{theorem}

Besides on the mentioned propositions, our proof for this theorem also depends on two lemmas: Lemma~\ref{lem:upper:probability:zero} from before, and the following simple consequence of triviality for variables.

\begin{lemma}\label{lem:triviality:and:partitions}
Consider any variable \(\zvar\colon\pspace\to\zvalues\) and any finite partition~\(\partition\) of~\(\zvalues\).
Assume that \(\zvar\) is trivial with respect to a coherent lower prevision~\(\lowprev\).
Then there is some \(\aevt[o]\in\partition\) such that \(\smash{\uppprev(\zvar^{-1}(\aevt))=0}\) for all \(\aevt\in\partition\setminus\set{\aevt[o]}\). 
\end{lemma}

\begin{proof}
If \(\uppprev(\zvar^{-1}(\aevt))=0\) for all \(\aevt\in\partition\), then we are done.
Without loss of generality, we may therefore assume that there is at least one \(\aevt[o]\in\partition\) such that \(\smash{\uppprev(\zvar^{-1}(\aevt[o]))}>0\). 
Hence, since the triviality of~\(\zvar\) implies the triviality of~\(\aevt[o]\), it must be that \(\uppprev(\zvar^{-1}(\co{\aevt[o]}))=0\). 
Consider now any \(\aevt\in\partition\setminus\set{\aevt[o]}\). 
Since \(\aevt\) and \(\aevt[o]\) are disjoint, we have that \(\aevt\subseteq\co{\aevt[o]}\). 
It therefore follows from coherence [use~\ref{ax:lowprev:infandsup} and~\ref{ax:lowprev:monotonicity}] that \(\smash{0\leq\uppprev(\zvar^{-1}(\aevt))\leq\uppprev(\zvar^{-1}(\co{\aevt[o]}))=0}\), so \(\smash{\uppprev(\zvar^{-1}(\aevt))=0}\).
\end{proof}

\begin{proof}\hspace{-1pt}{\bf{\itshape{of Theorem~\ref{theo:lower:equivalentdefinitions}}}}
Necessity is immediate from Proposition~\ref{prop:variableimpliesevents}, so it remains to prove sufficiency.
So let us assume that \(\xvar^{-1}(\aevt)\) is S-irrelevant to~\(\yvar^{-1}(\bevt)\) with respect to~\(\rejectset[\lowprev]\), for all \(\aevt\subseteq\xvalues\) and \(\bevt\subseteq\yvalues\). 
We need to prove that \(\xvar\) is S-irrelevant to~\(\yvar\) with respect to~\(\rejectset[\lowprev]\). 
We will do this using the criterion~\eqref{eq:S-irrelevance:lowprev:for:variables}, and consider, to this end, any partition~\(\partition\) of~\(\xvalues\) and any choice of simple gambles \(s_{\aevt}\in\sgbls[\yvalues]\) for \(\aevt\in\partition\). 
There are now two possible cases.

The first case we consider is that \(\xvar\) is trivial with respect to~\(\lowprev\). 
We then infer from Lemma~\ref{lem:triviality:and:partitions} that there is some \(\aevt[o]\in\partition\) such that \(\smash{\uppprev(\xvar^{-1}(\aevt))=0}\) for all \(\aevt\in\partition\setminus\set{\aevt[o]}\). 
Repeated invocation of Lemma~\ref{lem:upper:probability:zero} then guarantees that 
\begin{equation*}
\lowprev\group[\bigg]{\sum_{\cevt\in\partition\setminus\set{\aevt[o]}}\ind{\cevt}(\xvar)[s_{\aevt[o]}(\yvar)-s_{\cevt}(\yvar)]}=0,
\end{equation*}
so this case is dealt with.

Next, we consider the case where \(\xvar\) is not trivial with respect to~\(\lowprev\). 
In that case, it follows from the assumption and Proposition~\ref{prop:lowerprecise:variables} that \(\yvar\) has a precise distribution \(\dis[\yvar]\) with respect to~\(\lowprev\). 
For any \(\aevt\in\partition\) and any \(\altlinprev\in\linprevs(\lowprev)\), we then get that
\begin{align*}
\altlinprev\group[\bigg]{\sum_{\cevt\in\partition\setminus\set{\aevt}}\ind{\cevt}(\xvar)[s_{\aevt}(\yvar)-s_{\cevt}(\yvar)]}
&=\altlinprev\group[\bigg]{s_{\aevt}(\yvar)-\sum_{\cevt\in\partition}\ind{\cevt}(\xvar)s_{\cevt}(\yvar)}\\
&=\altlinprev\group{s_{\aevt}(\yvar)}-\sum_{\cevt\in\partition}\altlinprev\group[\big]{\ind{\cevt}(\xvar)s_{\cevt}(\yvar)}\\
&=\altlinprev\group{s_{\aevt}(\yvar)}-\sum_{\cevt\in\partition}\altlinprev\group[\big]{\ind{\cevt}(\xvar)}\altlinprev\group[\big]{s_{\cevt}(\yvar)}\\
&=\dis[\yvar](s_{\aevt})-\sum_{\cevt\in\partition}\altlinprev(\ind{\cevt}(\xvar))\dis[\yvar](s_{\cevt}),
\end{align*}
where the crucial third equality follows from Propositions~\ref{prop:loweronesided:variables} and~\ref{prop:independenceforvariables}.
Hence, for any \(\aevt\in\partition\):
\begin{align*}
\lowprev\group[\bigg]{\sum_{\cevt\in\partition\setminus\set{\aevt}}\ind{\cevt}(\xvar)[s_{\aevt}(\yvar)-s_{\cevt}(\yvar)]}
&=\min_{\altlinprev\in\linprevs(\lowprev)}
\group[\bigg]{\dis[\yvar](s_{\aevt})-\sum_{\cevt\in\partition}\altlinprev(\ind{\cevt}(\xvar))\dis[\yvar](s_{\cevt})}\\
&=\dis[\yvar](s_{\aevt})-\max_{\altlinprev\in\linprevs(\lowprev)}\sum_{\cevt\in\partition}\altlinprev(\ind{\cevt}(\xvar))\dis[\yvar](s_{\cevt})\\
&\geq\dis[\yvar](s_{\aevt})-\max_{\cevt\in\partition}\dis[\yvar](s_{\cevt}),
\end{align*}
where the inequality holds because \(\sum_{\cevt\in\partition}\altlinprev(\ind{\cevt}(\xvar))\dis[\yvar](s_{\cevt})\) is a convex combination of the terms \(\dis[\yvar](s_{\cevt})\), \(\cevt\in\partition\), and is therefore dominated by their maximum.
This tells us that
\begin{equation*}
\max_{\aevt\in\partition}\lowprev\group[\bigg]{\sum_{\cevt\in\partition\setminus\set{\aevt}}\ind{\cevt}(\xvar)[s_{\aevt}(\yvar)-s_{\cevt}(\yvar)]}
\geq\max_{\aevt\in\partition}\dis[\yvar](s_{\aevt})-\max_{\cevt\in\partition}\dis[\yvar](s_{\cevt})
=0,
\end{equation*}
as required.
\end{proof}

Theorem~\ref{theo:lower:equivalentdefinitions} generalises easily to general Archimedean sets of desirable option sets because for those, S-irrelevance can be expressed in terms of the representing lower previsions; see Propositions~\ref{prop:S-irrelevance:in:terms_of:binary} and~\ref{prop:S-irrelevance:in:terms_of:binary:variables}. 
This yields the following simple characterisation of S-irrelevance for variables in terms of S-irrelevance for events. 
It provides an {\itshape ex post} justification for our having focused on the latter first, and for having paid so much attention to it in Section~\ref{sec:S-irrelevance:for:events}.

\begin{theorem}\label{theo:lower:equivalentdefinitions:K}
Let \(\rejectset\) be an Archimedean set of desirable option sets and consider two variables \(\xvar\) and \(\yvar\). 
Then \(\xvar\) is S-irrelevant to~\(\yvar\) with respect to~\(\rejectset\) if and only if, for all \(\aevt\subseteq\xvalues\) and \(\bevt\subseteq\yvalues\), \(\xvar^{-1}(\aevt)\) is S-irrelevant to~\(\yvar^{-1}(\bevt)\) with respect to~\(\rejectset\).
\end{theorem}
\begin{proof}
Immediate from Theorem~\ref{theo:lower:equivalentdefinitions} and
Propositions~\ref{prop:S-irrelevance:in:terms_of:binary} and~\ref{prop:S-irrelevance:in:terms_of:binary:variables}.
\end{proof}

The characterisation of S-irrelevance for variables in terms of S-irrelevance for events in Theorem~\ref{theo:lower:equivalentdefinitions} also leads to a fairly easily proven generalisation of Proposition~\ref{prop:triviality:implies:S:independence}, describing the implications of triviality.

\begin{proposition}\label{prop:triviality:implies:S:independence:variables}
If the variables \(\xvar\) or \(\yvar\) are trivial with respect to a coherent lower prevision~\(\lowprev\), then \(\xvar\) and \(\yvar\) are S-independent with respect to~\(\rejectset[\lowprev]\).
\end{proposition}

\begin{proof}
Assume that \(\xvar\) or \(\yvar\) is trivial with respect to~\(\lowprev\).
Consider any \(\aevt\subseteq\xvalues\) and \(\bevt\subseteq\yvalues\). 
If the variable \(\xvar\) is trivial with respect to~\(\lowprev\), then the event \(\xvar^{-1}(\aevt)\) is trivial with respect to~\(\lowprev\).
Similarly, if \(\yvar\) is trivial with respect to~\(\lowprev\), then \(\yvar^{-1}(\bevt)\) is as well. 
In both cases, it follows from Proposition~\ref{prop:triviality:implies:S:independence} that \(\xvar^{-1}(\aevt)\) and \(\yvar^{-1}(\bevt)\) are S-independent with respect to~\(\rejectset[\lowprev]\). 
Since this is true for every \(\aevt\subseteq\xvalues\) and \(\bevt\subseteq\yvalues\), it follows from Theorem~\ref{theo:lower:equivalentdefinitions} that \(\xvar\) and \(\yvar\) are S-independent with respect to~\(\rejectset[\lowprev]\).
\end{proof}

All this preparatory work is about to bear fruit in the final two theorems of this section.
The following characterisation provides better insight into what---and how surprisingly strong---the implications of an S-irrelevance assessment really are.

\begin{theorem}\label{theo:lowertwosided:variables}
Consider a coherent lower prevision~\(\lowprev\) on~\(\gbls\) and two variables \(\xvar\) and \(\yvar\).
Then \(\xvar\) is S-irrelevant to~\(\yvar\) with respect to~\(\rejectset[\lowprev]\) if and only if \(\xvar\) is trivial with respect to~\(\lowprev\), or if \(\yvar\) has a precise distribution \(\dis[\yvar]\) with respect to~\(\lowprev\) and
\begin{equation}
\label{eq:lower:twosided:irrelevance:variables}
\lowuppprev(\gbl(\xvar)\altgbl(\yvar))=\lowuppprev(\gbl(\xvar))\odot\dis[\yvar](\altgbl)
\text{ for all gambles \(\gbl\) on \(\xvalues\) and \(\altgbl\) on \(\yvalues\).}%
\footnote{Since `factorisation' of this kind for multiple variables leads to a version of the law of large numbers \cite{cooman2004a,cooman2011a}, it doesn't seem too farfetched to envision extensions of S-irrelevance and S-independence from two to multiple variables that allow us to prove similar laws of large numbers.}
\end{equation}
Similarly, \(\xvar\) and \(\yvar\) are S-independent with respect to~\(\rejectset[\lowprev]\) if and only if \(\xvar\) or \(\yvar\) are trivial with respect to~\(\lowprev\), or if they both have precise distributions \(\dis[\xvar]\) and \(\dis[\yvar]\) with respect to~\(\lowprev\) and
\begin{equation}\label{eq:lower:twosided:independence:variables}
\lowuppprev(\gbl(\xvar)\altgbl(\yvar))=\dis[\xvar](\gbl)\dis[\yvar](\altgbl)
\text{ for all gambles \(\gbl\) on \(\xvalues\) and \(\altgbl\) on \(\yvalues\).}
\end{equation}
\end{theorem}

\begin{proof}
We begin with the first statement.
For necessity, assume that \(\xvar\) is non-trivial with respect to~\(\lowprev\) and S-irrelevant to~\(\yvar\) with respect to~\(\rejectset[\lowprev]\). 
Theorem~\ref{theo:lower:equivalentdefinitions} and Proposition~\ref{prop:lowerprecise:variables} then guarantee that \(\yvar\) has a precise distribution \(\linprev[\yvar]\) with respect to~\(\lowprev\).
Consider now any \(\gbl\in\gbls(\xvalues)\), \(\altgbl\in\gbls(\yvalues)\) and \(\altlinprev\in\linprevs(\lowprev)\), then on the one hand \(\altlinprev(\altgbl(\yvar))=\dis[\yvar](\altgbl)\) and on the other hand \(\altlinprev(\gbl(\xvar)\altgbl(\yvar))=\altlinprev(\gbl(\xvar))\altlinprev(\altgbl(\yvar))=\altlinprev(\gbl(\xvar))\dis[\yvar](\altgbl)\) by Theorem~\ref{theo:lower:equivalentdefinitions} and Propositions~\ref{prop:loweronesided:variables} and~\ref{prop:independenceforvariables}.
Hence by taking minima and maxima over all \(\altlinprev\in\linprevs(\lowprev)\) on both sides, we get that
\begin{multline*}
\lowprev(\gbl(\xvar)\altgbl(\yvar))=
\begin{cases}
\lowprev(\gbl(\xvar))\dis[\yvar](\altgbl)&\text{if \(\dis[\yvar](\altgbl)\geq0\)}\\
\uppprev(\gbl(\xvar))\dis[\yvar](\altgbl)&\text{if \(\dis[\yvar](\altgbl)\leq0\)}
\end{cases}\\
\text{ and }
\uppprev(\gbl(\xvar)\altgbl(\yvar))=
\begin{cases}
\uppprev(\gbl(\xvar))\dis[\yvar](\altgbl)&\text{if \(\dis[\yvar](\altgbl)\geq0\)}\\
\lowprev(\gbl(\xvar))\dis[\yvar](\altgbl)&\text{if \(\dis[\yvar](\altgbl)\leq0\)}
\end{cases}
\end{multline*}
which can indeed be summarised as \(\lowuppprev(\gbl(\xvar)\altgbl(\yvar))=\lowuppprev(\gbl(\xvar))\odot\dis[\yvar](\altgbl)\).

We now turn to sufficiency. 
If \(\xvar\) is trivial with respect to~\(\lowprev\), it follows immediately from Proposition~\ref{prop:triviality:implies:S:independence:variables} that \(\xvar\) is S-irrelevant to~\(\yvar\) with respect to~\(\rejectset[\lowprev]\). 
We can therefore assume, without loss of generality, that \(\yvar\) has a precise distribution \(\dis[\yvar]\) with respect to~\(\lowprev\) and that Equation~\eqref{eq:lower:twosided:irrelevance:variables} holds. 
Consider now any \(\aevt\subseteq\xvalues\) and \(\bevt\subseteq\yvalues\). 
Then for any \(\gbl(\xvar)\in\gbls[\xvar^{-1}(\aevt)]\) and \(\altgbl(\yvar)\in\gbls[\yvar^{-1}(\bevt)]\), we have that \(\smash{\lowuppprev(\altgbl(\yvar))=\dis[\yvar](\altgbl)}\) because \(\yvar\) has a precise distribution with respect to~\(\lowprev\), and that \(\smash{\lowuppprev(\gbl(\xvar)\altgbl(\yvar))=\lowuppprev(\gbl(\xvar))\odot\dis[\yvar](\altgbl)}\) because of Equation~\eqref{eq:lower:twosided:irrelevance:variables}. 
It therefore follows from Theorem~\ref{theo:lowertwosided} that \(\xvar^{-1}(\aevt)\) is S-irrelevant to~\(\yvar^{-1}(\bevt)\). 
Since \(\aevt\subseteq\xvalues\) and \(\bevt\subseteq\yvalues\) were arbitrary, it follows from Theorem~\ref{theo:lower:equivalentdefinitions} that \(\xvar\) is S-irrelevant to~\(\yvar\) with respect to~\(\rejectset[\lowprev]\).

Next, we turn to the second statement.
For necessity, assume that \(\xvar\) and \(\yvar\) are non-trivial with respect to~\(\lowprev\) and S-independent with respect to~\(\rejectset[\lowprev]\). 
Theorem~\ref{theo:lower:equivalentdefinitions} and Proposition~\ref{prop:lowerprecise:variables} then guarantee that \(\xvar\) and \(\yvar\) respectively have precise distributions \(\dis[\xvar]\) and \(\dis[\yvar]\) with respect to~\(\lowprev\).
Consider now any \(\gbl\in\gbls(\xvalues)\), \(\altgbl\in\gbls(\yvalues)\) and \(\altlinprev\in\linprevs(\lowprev)\), then on the one hand \(\altlinprev(\gbl(\xvar))=\dis[\xvar](\gbl)\) and \(\altlinprev(\altgbl(\yvar))=\dis[\yvar](\altgbl)\) and on the other hand \(\altlinprev(\gbl(\xvar)\altgbl(\yvar))=\altlinprev(\gbl(\xvar))\altlinprev(\altgbl(\yvar))=\dis[\xvar](\gbl)\dis[\yvar](\altgbl)\) by Theorem~\ref{theo:lower:equivalentdefinitions} and Propositions~\ref{prop:loweronesided:variables} and~\ref{prop:independenceforvariables}.
Hence by taking minima and maxima over all \(\altlinprev\in\linprevs(\lowprev)\) on both sides, we get that, indeed, \(\smash{\lowuppprev(\gbl(\xvar)\altgbl(\yvar))=\dis[\xvar](\gbl)\dis[\yvar](\altgbl)}\).

We now turn to sufficiency.
If \(\xvar\) or \(\yvar\) are trivial with respect to~\(\lowprev\), then it follows immediately from Proposition~\ref{prop:triviality:implies:S:independence:variables} that \(\xvar\) and \(\yvar\) are S-independent with respect to~\(\rejectset[\lowprev]\). 
We can therefore assume, without loss of generality, that \(\xvar\) and \(\yvar\) respectively have precise distributions \(\dis[\xvar]\) and \(\dis[\yvar]\) with respect to~\(\lowprev\) and that Equation~\eqref{eq:lower:twosided:independence:variables} holds, and prove that \(\xvar\) and \(\yvar\) are S-independent with respect to~\(\rejectset[\lowprev]\). 
Since Equation~\eqref{eq:lower:twosided:independence:variables} implies Equation~\eqref{eq:lower:twosided:irrelevance:variables}, the S-irrelevance of~\(\xvar\) to \(\yvar\) with respect to~\(\rejectset[\lowprev]\) follows from the first part of this theorem. 
Since Equation~\eqref{eq:lower:twosided:independence:variables} is symmetric in~\(\xvar\) and \(\yvar\), the S-irrelevance of~\(\yvar\) to \(\xvar\) with respect to~\(\rejectset[\lowprev]\) follows in exactly the same way. 
Hence, we find that, indeed, \(\xvar\) and \(\yvar\) are S-independent with respect to~\(\rejectset[\lowprev]\).
\end{proof}

By combining this result with Proposition~\ref{prop:S-irrelevance:in:terms_of:binary:variables}, we immediately obtain characterisations for S-irrelevance and S-independence for Archimedean sets of desirable option sets~\(\rejectset\), in terms of their representing lower previsions. 
If we ignore the trivial cases, we see that each of these lower previsions features both precision and factorisation.
As we did in Section~\ref{sec:lowerprevisions:events}, we now seek to exclude the trivial cases by imposing credible indeterminacy, this time for variables instead of events.
We say that the variable~\(\zvar\) is \emph{credibly indeterminate} with respect to a coherent set of desirable option sets~\(\rejectset\) if there is at least one event \(\cevt\subseteq\zvalues\) such that \(\zvar^{-1}(\cevt)\) is credibly indeterminate with respect to~\(\rejectset\). 
If \(\rejectset\) is Archimedean, then due to Proposition~\ref{prop:credible:in:terms:of:lowprev}, this means that there is some \(\epsilon>0\) such that for all \(\lowprev\in\cohlowprevs(\rejectset)\), both \(\lowprev[\zvar](\cevt)>\epsilon\) and \(\lowprev[\zvar](\co{\cevt})>\epsilon\).

Similarly to what we found for events, the condition of credible indeterminacy, when combined with S-independence, allows us to infer both precision---for~\(\xvar\) and~\(\yvar\)---and factorisation for every representing lower prevision of an Archimedean set of desirable option sets~\(\rejectset\), \emph{without having to impose mixingness}. 
This is a surprisingly strong implication, we think, and especially so since credible indeterminacy for a variable \(\zvar\) is such a weak requirement, as it only requires one single event about this variable \(\zvar\) to be credibly indeterminate. 

\begin{theorem}\label{theo:loweronesided:K}
Let \(\rejectset\) be an Archimedean set of desirable option sets and consider two variables \(\xvar\) and \(\yvar\). 
If \(\xvar\) is credibly indeterminate and S-irrelevant to~\(\yvar\) with respect to~\(\rejectset\), then for all \(\lowprev\in\cohlowprevs(\rejectset)\), \(\yvar\) has a precise distribution~\(\dis[\yvar]\) with respect to~\(\lowprev\) and
\begin{equation*}
\lowuppprev(\gbl(\xvar)\altgbl(\yvar))=\lowuppprev(\gbl(\xvar))\odot\dis[\yvar](\altgbl)
\text{ for all gambles \(\gbl\) on \(\xvalues\) and \(\altgbl\) on \(\yvalues\).}
\end{equation*}
Similarly, if \(\xvar\) and \(\yvar\) are credibly indeterminate and S-independent with respect to~\(\rejectset\), then for all \(\lowprev\in\cohlowprevs(\rejectset)\), \(\xvar\) and \(\yvar\) have a precise distribution~\(\dis[\xvar]\) and \(\dis[\yvar]\) with respect to~\(\lowprev\), respectively, and
\begin{equation*}
\lowuppprev(\gbl(\xvar)\altgbl(\yvar))=\dis[\xvar](\gbl)\dis[\yvar](\altgbl)
\text{ for all gambles \(\gbl\) on \(\xvalues\) and \(\altgbl\) on \(\yvalues\).}
\end{equation*}
\end{theorem}
\begin{proof}
Due to Proposition~\ref{prop:S-irrelevance:in:terms_of:binary:variables} and Theorem~\ref{theo:lowertwosided:variables}, it suffices to show that the credible indeterminacy of \(\xvar\) implies that \(\xvar\) is non-trivial with respect to every \(\lowprev\in\cohlowprevs(\rejectset)\). 
So assume that \(\xvar\) is credibly indeterminate with respect to~\(\rejectset\). 
Then there is some \(\aevt\subseteq\xvalues\) that is credibly indeterminate with respect to~\(\rejectset\), meaning that \(\taevt\) is credible for each \(\taevt\in\set{\aevt,\co{\aevt}}\). 
For any \(\lowprev\in\cohlowprevs(\rejectset)\), it then follows from Proposition~\ref{prop:credible:in:terms:of:lowprev} that there is some \(\epsilon>0\) such that \(\lowprev(\taevt)>\epsilon\). 
Hence, since \(\smash{\uppprev(\taevt)\geq\lowprev(\taevt)>\epsilon>0}\), we see that \(\aevt\) is non-trivial with respect to~\(\lowprev\), implying that \(\xvar\) is non-trivial with respect to~\(\lowprev\) as well.
\end{proof}

\subsection{S-irrelevance for variables with respect to linear prevision models}
We now want to reward those readers who are fans of decision-making with linear previsions---or precise probability models---and who have nevertheless had the courage and determination to follow our arguments all the way to this point.
Due to the heavy lifting already done for the more general cases of lower previsions and Archimedean models in the previous section, we are now able, without further ado, to present our results for the special case of linear previsions, in Theorem~\ref{theo:variables:and:events:linear:previsions}, and for the more involved, non-binary case of mixing models, in Theorem~\ref{theo:variables:and:events:mixing:models} below.

Observe, first of all, that Theorem~\ref{theo:lower:equivalentdefinitions} also applies in particular in the linear previsions context of the present section. 
It allows us to apply arguments for events---Theorem~\ref{theo:lineartwosided} in  particular---in order to obtain the following results about variables in a fairly straightforward manner.

\begin{theorem}\label{theo:variables:and:events:linear:previsions}
Consider two variables~\(\xvar\) and~\(\yvar\) and a linear prevision~\(\linprev\) on~\(\gbls\).
Then the following statements are equivalent:
\begin{enumerate}[label=\upshape(\roman*),leftmargin=*]
\item\label{it:variables:and:events:variables:irrelevance:linear} \(\xvar\) is S-irrelevant to~\(\yvar\) with respect to~\(\rejectset[\linprev]\);
\item\label{it:variables:and:events:variables:independence:linear} \(\xvar\) and \(\yvar\) are S-independent with respect to~\(\rejectset[\linprev]\);
\item\label{it:variables:and:events:determinate:or:precise:independence:factorisation}
\(\xvar\) and \(\yvar\) are independent with respect to~\(\linprev\).
\end{enumerate}
\end{theorem}

\begin{proof}
Since \(\linprev\) is a linear prevision and hence definitely a coherent lower prevision, it follows from Theorem~\ref{theo:lower:equivalentdefinitions} that condition~\ref{it:variables:and:events:variables:irrelevance:linear} holds if and only if for all \(\aevt\subseteq\xvalues\) and \(\bevt\subseteq\yvalues\), \(\smash{\xvar^{-1}(\aevt)}\) is S-irrelevant to~\(\smash{\yvar^{-1}(\bevt)}\) with respect to~\(\rejectset[\linprev]\).

Similarly, condition~\ref{it:variables:and:events:variables:independence:linear} holds if and only if for all \(\aevt\subseteq\xvalues\) and \(\bevt\subseteq\yvalues\), \(\smash{\xvar^{-1}(\aevt)}\) is S-independent to \(\smash{\yvar^{-1}(\bevt)}\) with respect to~\(\rejectset[\linprev]\).

Furthermore, because of Definition~\ref{def:independenceforvariables}, condition~\ref{it:variables:and:events:determinate:or:precise:independence:factorisation} holds if and only if for all \(\aevt\subseteq\xvalues\) and \(\bevt\subseteq\yvalues\), the events \(\xvar^{-1}(\aevt)\) and \(\yvar^{-1}(\bevt)\) are independent with respect to~\(\linprev\).

Given these observations, the equivalence of \ref{it:variables:and:events:variables:irrelevance:linear}, \ref{it:variables:and:events:variables:independence:linear} and \ref{it:variables:and:events:determinate:or:precise:independence:factorisation} follows immediately from Theorem~\ref{theo:lineartwosided}.
\end{proof}

Since we know from Theorem~\ref{thm:archimedean:and:mixing} that Archimedean and mixing models correspond to sets of linear previsions, the result above can be extended to Archimedean and mixing models too. 
Observe that in this case, due to the mixingness property, credible indeterminacy is not required for factorisation to appear---here in the form of independence; see Proposition~\ref{prop:independenceforvariables}. Note also that, as a direct result of Theorem~\ref{theo:lower:equivalentdefinitions:K}, the conditions (i) and (ii) can be equivalently expressed in terms of events as well.

\begin{theorem}\label{theo:variables:and:events:mixing:models}
Consider two variables~\(\xvar\) and~\(\yvar\) and an Archimedean and mixing set of desirable option sets~\(\rejectset\).
Then the following statements are equivalent: 
\begin{enumerate}[label=\upshape(\roman*),leftmargin=*]
\item \(\xvar\) is S-irrelevant to~\(\yvar\) with respect to~\(\rejectset\);
\item \(\xvar\) and \(\yvar\) are S-independent with respect to~\(\rejectset\);
\item \(\xvar\) and \(\yvar\) are independent with respect to~\(\linprev\), for all \(\linprev\in\linprevs(\rejectset)\).
\end{enumerate}
\end{theorem}

\begin{proof}
This result follows directly from Theorem~\ref{theo:variables:and:events:linear:previsions} and Proposition~\ref{prop:S-irrelevance:in:terms_of:binary:variables}.
\end{proof}

\section{The far-reaching implications of S-irrelevance and S-independence}\label{sec:discussion}
After the detailed mathematical analysis of the previous sections, let us now take a moment to consider what these mathematical results imply, and how far-reaching we believe these implications to be. 
In doing so, we will also lay the foundations for talking about inferences and decisions involving variables and non-binary choice models.

If our subject has a choice model \(\rejectset\) for choosing between gambles on~\(\pspace\), we can derive from that her choice model \(\rejectset[\zvar]\) for choosing between gambles on the value of a variable \(\zvar\colon\pspace\to\zvalues\).
We will use the following (notational) device: for any option set \(\altoptsettoo\in\optsets(\zvalues)\) of gambles on the possibility space~\(\zvalues\), we let
\begin{equation*}
\altoptsettoo(\zvar)\coloneqq\cset{\altgbltoo(\zvar)}{\altgbltoo\in\altoptsettoo}\in\optsets(\pspace)
\end{equation*}
be the corresponding option set of gambles~\(h(\zvar)\coloneqq h\circ\zvar\) on the variable~\(\zvar\), which are, of course, gambles whose domain is the possibility space~\(\pspace\).
Then clearly,
\begin{equation*}
\rejectset[\zvar]
\coloneqq\cset{\altoptsettoo\in\optsets(\zvalues)}{\altoptsettoo(\zvar)\in\rejectset}
\end{equation*}
is the set of desirable option sets on \(\zvalues\) that represents the choices between gambles that depend on the variable \(\zvar\), implicit in~\(\rejectset\).
It is completely in the spirit of the existing terminology in standard probability theory to call this choice model \(\rejectset[\zvar]\) the \emph{distribution} of the variable \(\zvar\), as it is a full decision-theoretic model for the subject's uncertainty about the value that \(\zvar\) assumes in~\(\zvalues\).

It is also a matter of simple and direct verification that this operation preserves coherence, mixingness and Archimedeanity.
Moreover, if \(\lowprev\) is a coherent lower prevision on \(\gbls(\pspace)\), then this operation turns the Archimedean \(\rejectset=\rejectset[\lowprev]\) into the Archimedean \(\rejectset[\zvar]=\rejectset[{\lowdis[\zvar]}]\), where the coherent lower prevision~\(\lowdis[\zvar]\) on~\(\gbls(\zvalues)\) is the lower distribution of~\(\zvar\) with respect to~\(\lowprev\), introduced in Section~\ref{sec:irrelevance:variables:lower:previsions}.
The same goes for a linear prevision~\(\linprev\) and the corresponding precise distribution~\(\dis[\zvar]\).
We prove some of these claims involving Archimedean (and mixing) models explicitly in the following proposition.

\begin{proposition}\label{prop:distributions:and:marginals}
Consider an Archimedean set of desirable option sets~\(\rejectset\) and a variable~\(\zvar\).
Then \(\rejectset[\zvar]\) is Archimedean too, and has \(\cset{\lowdis[\zvar]}{\lowprev\in\cohlowprevs(\rejectset)}\) as a set of representing coherent lower previsions. 
If \(\rejectset\) is furthermore mixing, then \(\rejectset[\zvar]\) is Archimedean and mixing, and has \(\cset{\dis[\zvar]}{\linprev\in\linprevs(\rejectset)}\) as a set of representing linear previsions.
\end{proposition}

\begin{proof}
First assume that \(\rejectset\) is Archimedean. Definition~\ref{def:archimedeanity} then implies that we can consider the following chain of equivalences for any \(\altoptsettoo\) in \(\optsets(\zvalues)\):
\begin{align*}
\altoptsettoo\in\rejectset[\zvar]
\ifandonlyif\altoptsettoo(\zvar)\in\rejectset
&\ifandonlyif(\forall\lowprev\in\cohlowprevs(\rejectset))\altoptsettoo(\zvar)\in\rejectset[\lowprev]\\
&\ifandonlyif(\forall\lowprev\in\cohlowprevs(\rejectset))(\exists\altgbltoo\in\altoptsettoo)\lowprev(\altgbltoo(\zvar))>0\\
&\ifandonlyif(\forall\lowprev\in\cohlowprevs(\rejectset))(\exists\altgbltoo\in\altoptsettoo)\lowdis(\altgbltoo)>0\\
&\ifandonlyif(\forall\lowprev\in\cohlowprevs(\rejectset))\altoptsettoo\in\rejectset[{\lowdis[\zvar]}].
\end{align*}
So \(\rejectset[\zvar]\) is indeed Archimedean, and has \(\cset[\big]{\lowdis[\zvar]}{\lowprev\in\cohlowprevs(\rejectset)}\) as a set of representing coherent lower previsions. 
If \(\rejectset\) is furthermore mixing, we know from Theorem~\ref{thm:archimedean:and:mixing} that \(\cohlowprevs(\rejectset)=\linprevs(\rejectset)\), which implies that \(\cset[\big]{\lowdis[\zvar]}{\lowprev\in\cohlowprevs(\rejectset)}=\cset[\big]{\dis[\zvar]}{\linprev\in\linprevs(\rejectset)}\). 
A second application of Theorem~\ref{thm:archimedean:and:mixing} therefore implies that \(\rejectset[\zvar]\) is indeed mixing.
\end{proof}

In order to explore the implications of what we have discovered in the previous sections, let us now focus on a decision problem involving gambles that depend on two variables~\(\xvar\) and~\(\yvar\).
These are gambles of the type \(\altgbltoo(\xvar,\yvar)\coloneqq\altgbltoo\circ(\xvar,\yvar)\), where \(\altgbltoo\) is some gamble on~\(\xyvalues\).
Of course \((\xvar,\yvar)\) can be seen as a new variable \(\xyvar\colon\pspace\to\xyvalues\colon\omega\mapsto(\xvar(\omega),\yvar(\omega))\), and all we have said above about the distribution of a variable can also be brought to bear on this variable \((\xvar,\yvar)\).
In particular, if our subject has a coherent set of desirable option sets~\(\rejectset\), then the so-called \emph{joint distribution} \(\rejectset[\xyvar]\) of the variable \(\xyvar\) is given by
\begin{equation*}
\rejectset[\xyvar]
\coloneqq\cset{\altoptsettoo\in\optsets(\xyvalues)}{\altoptsettoo\xyvar\in\rejectset},
\end{equation*}
and, of course, choices between gambles on the value of, say, \(\yvar\) separately are modelled by the so-called \emph{marginal distribution}~\(\rejectset[\yvar]\) of~\(\yvar\), given by\footnote{This notation uses the implicit convention that gambles with domain~\(\yvalues\) are considered as special instances of gambles with domain~\(\xyvalues\).}
\begin{equation*}
\rejectset[\yvar]
\coloneqq\cset{\optset\in\optsets(\yvalues)}{\optset(\yvar)\in\rejectset}
=\rejectset[\xyvar]\cap\optsets(\yvalues),
\end{equation*}
where the rightmost equality also indicates that the marginal distribution~\(\rejectset[\yvar]\) can be derived from the joint distribution~\(\rejectset[\xyvar]\) by a \emph{marginalisation operation}, similarly to what is done for sets of desirable gambles \cite{debock2015:thesis,debock2015:credal:nets,cooman:2012:indnatexdesirs}.

Assume now that our subject's choice model~\(\rejectset\) is Archimedean---but not necessarily mixing---and that she has furthermore made the assessment that \(\xvar\) is S-irrelevant to~\(\yvar\), and that this assessment is reflected in her Archimedean model \(\rejectset\).
We will also assume that \(\rejectset\) reflects her beliefs that \(\xvar\) is credibly indeterminate, which we have argued is a rather weak requirement to impose.
The strong consequences of these assumptions have been derived in Theorem~\ref{theo:loweronesided:K}. 
In particular, it guarantees factorisation properties for the binary distributions of \(\xvar\) and \(\yvar\) in the representation of the Archimedean model~\(\rejectset\): the distribution~\(\rejectset[\xyvar]\) is represented by a set of factorising lower previsions with precise (linear) marginals for \(\yvar\). 
And if we furthermore symmetrise the assessment of our subject---that is, if \(\xvar\) and \(\yvar\) are (both) credibly indeterminate and S-independent---then the same is true for the marginals of \(\xvar\).

What is perhaps the most striking about Theorem~\ref{theo:loweronesided:K}, however, are its implications for the choice model~\(\rejectset[\yvar]\). 
Since \(\rejectset\) is Archimedean, we know from Proposition~\ref{prop:distributions:and:marginals} that \(\rejectset[\yvar]\) is Archimedean as well. 
What is very surprising, though, is that our subject's assessments---that \(\xvar\) is credibly indeterminate and S-irrelevant to \(\yvar\)---imply that it must be \emph{mixing} as well.

\begin{corollary}\label{corol:crazyresult}
Suppose that a variable \(\xvar\) is credibly indeterminate and S-irrelevant to a variable~\(\yvar\) with respect to an Archimedean set of desirable option sets~\(\rejectset\).
Then the distribution~\(\rejectset[\yvar]\) of~\(\yvar\) is an Archimedean and mixing set of desirable option sets.
\end{corollary}

\begin{proof}
For any \(\altoptset\in\optsets(\yvalues)\), we have that
\begin{align*}
\altoptset\in\rejectset[\yvar]
\ifandonlyif(\forall\lowprev\in\cohlowprevs(\rejectset))\altoptset\in\rejectset[{\lowdis[\yvar]}]
\ifandonlyif(\forall\lowprev\in\cohlowprevs(\rejectset))\altoptset\in\rejectset[{\dis[\yvar]}],
\end{align*}
where the first equivalence follows from Proposition~\ref{prop:distributions:and:marginals}, and the second from Theorem~\ref{theo:loweronesided:K}. 
Since we know from Theorem~\ref{theo:rejectsets:representation:lowerprev:twosided} that \(\cohlowprevs(\rejectset)\) is non-empty, it therefore follows from Theorem~\ref{thm:archimedean:and:mixing} that \(\rejectset[\yvar]\) is Archimedean and mixing. 
\end{proof}
\noindent
If the Archimedean set of desirable option sets~\(\rejectset\) is mixing, then the mixingness of~\(\rejectset[\yvar]\) follows easily from the fact that mixingness is preserved under marginalisation; see for example Proposition~\ref{prop:distributions:and:marginals}. 
The striking thing about Corollary~\ref{corol:crazyresult} is that it doesn't require \(\rejectset\) to be mixing: we obtain the mixingness of~\(\rejectset[\yvar]\) using only credible indeterminacy and S-irrelevance.

Since mixing sets of desirable options sets correspond to choice functions governed by E-admissibility, the implications of this result are far-reaching: we find that choices between gambles that depend only on~\(\yvar\) will be governed by E-admissibility with respect to a set of linear previsions, for example \(\linprevs(\rejectset[\yvar])\). 
This is a very curious and amazingly strong result.
As soon as a subject assumes that there is \emph{some} credibly indeterminate variable \(\xvar\) that is S-irrelevant to a variable~\(\yvar\), which seems a very weak assumption to make, she is forced by coherence---and Archimedeanity---to use a \emph{mixing} model for \(\yvar\), and to use E-admissibility as her decision scheme for choosing between gambles on \(\yvar\). 
To give a simple example, we believe that our flipping a coin here in Ghent today will not affect in any way whether Teddy will have pickled herring for breakfast tomorrow morning.
As soon as we translate this belief into an assessment that the outcome of our coin flip today is credibly indeterminate---which seems uncontroversial for a coin flip---and \mbox{S-irrelevant} to Teddy's choice of breakfast tomorrow, we are forced by coherence---and Archimedeanity---to use a mixing model for our uncertainty about Teddy's breakfast choice.

It would seem, then, that our mathematical derivations in this paper lead to an argument in favour of using mixing models and decision schemes based on E-admissibility. 
It is indeed very easy to imagine that there are experiments whose outcomes---variables~\(\xvar\)---are indeterminate and have nothing whatsoever to do with the outcome---variable~\(\yvar\)---of the experiment that we are currently considering. 
As soon as we translate this `being indeterminate and having nothing whatsoever to do with' by an assessment of credible indeterminacy and S-irrelevance, we are led to using mixing models and E-admissibility only.

We don't want to take this discussion too far, but still feel inclined to suggest that, perhaps, it is the translation that constitutes the Achilles' heel of this argument.
Going back to binary variables, or events, in the interest of simplicity, isn't requiring that the composite gamble~\(\ind{\aevt}\gbl+\ind{\co{\aevt}}\altgbl-\epsilon\) must \emph{always} be \emph{rejected} from the set \(\set{\gbl,\altgbl,\ind{\aevt}\gbl+\ind{\co{\aevt}}\altgbl-\epsilon}\) for \emph{all} \(\gbl,\altgbl\in\gbls[\bevt]\) too strong if we want to express that `whether the event \(\aevt\) obtains has nothing whatsoever to do with whether \(\bevt\) obtains'?
At least one of us isn't entirely convinced of the validity of this requirement.
Even if our subject believes that the event~\(\aevt\) has no effect on the event~\(\bevt\), why should she then reject the gamble \(\ind{\aevt}\gbl+\ind{\co{\aevt}}\altgbl-\epsilon\) from the set \(\set{\gbl,\altgbl,\ind{\aevt}\gbl+\ind{\co{\aevt}}\altgbl-\epsilon}\)? 
Or equivalently, why should she then necessarily find \(\ind{\co{\aevt}}(\gbl-\altgbl)+\epsilon\) or \(\ind{\aevt}(\altgbl-\gbl)+\epsilon\) desirable? 
For example, if \(\gbl\) and \(\altgbl\) are deemed incomparable by our subject, meaning that \(\gbl-\altgbl\) nor \(\altgbl-\gbl\) are desirable, what would then compel her to find \(\ind{\aevt}\gbl+\ind{\co{\aevt}}\altgbl-\epsilon\) comparable to---even dominated by---\(f\) or \(g\). Or to rephrase it one more time: if \(\gbl-\altgbl\) nor \(\altgbl-\gbl\) are deemed desirable, why then should \(\ind{\co{\aevt}}(\gbl-\altgbl)+\epsilon\) or \(\ind{\aevt}(\altgbl-\gbl)+\epsilon\) be desirable? 
We definitely think that these and related questions merit further attention.

\begin{acknowledgement}
We would like to thank Teddy Seidenfeld for the many discussions, throughout the years, on so many issues related to imprecise probabilities and the foundations of decision-making.
This paper, and our related earlier work on choice functions, would not have existed without his constructive and destructive criticism of our earlier work on binary choice. 

We would also like to thank the editors of this Festschrift for giving us the opportunity to contribute to it, and two anonymous reviewers for their valuable and constructive feedback.

Jasper De Bock's work was partially supported by his BOF Starting Grant ``Rational decision making under uncertainty: a new paradigm based on choice functions'', number 01N04819.

As with most of our joint work, there is no telling, after a while, which of us two had what idea, or did what, exactly. 
An irrelevant coin flip may have determined the actual order we are listed in.
\end{acknowledgement}

\bibliographystyle{plain}
\bibliography{general}

\begin{thebibliography}{10}

\bibitem{augustin2013:itip}
Thomas Augustin, Frank P.~A. Coolen, Gert {d}e Cooman, and Matthias C.~M.
  Troffaes, editors.
\newblock {\em Introduction to Imprecise Probabilities}.
\newblock John Wiley \& Sons, 2014.

\bibitem{couso2011}
In\'es Couso and Seraf\'{\i}n Moral.
\newblock Sets of desirable gambles: conditioning, representation, and precise
  probabilities.
\newblock {\em International Journal of Approximate Reasoning},
  52(7):1034--1055, 2011.

\bibitem{debock2015:thesis}
Jasper De~Bock.
\newblock {\em Credal Networks under Epistemic Irrelevance: Theory and
  Algorithms}.
\newblock PhD thesis, Ghent University, Faculty of Engineering and
  Architecture, 2015.

\bibitem{ipmu2020debock}
Jasper De~Bock.
\newblock {A}rchimedean choice functions.
\newblock In {\em Information Processing and Management of Uncertainty in
  Knowledge-Based Systems (Proceedings of IPMU 2020)}, pages 195--209. Springer
  International Publishing, 2020.

\bibitem{ipmu2020debock:arxiv}
Jasper De~Bock.
\newblock Archimedean choice functions: an axiomatic foundation for imprecise
  decision making.
\newblock 2020.
\newblock ArXiv e-print: 2002.05196.

\bibitem{partialorderchoice2020debock:arxiv}
Jasper De~Bock.
\newblock Choice functions based on sets of strict partial orders: an axiomatic
  characterisation.
\newblock 2020.
\newblock ArXiv e-print: 2003.11631.

\bibitem{debock2015:credal:nets}
Jasper De~Bock and Gert de~Cooman.
\newblock Credal networks under epistemic irrelevance: The sets of desirable
  gambles approach.
\newblock {\em International Journal of Approximate Reasoning}, 56(part
  A):178--207, 2015.

\bibitem{debock2018}
Jasper De~Bock and Gert {d}e Cooman.
\newblock A desirability-based axiomatisation for coherent choice functions.
\newblock In {\em Uncertainty Modelling in Data Science (Proceedings of SMPS
  2018)}, pages 46--53, 2018.

\bibitem{debock2018:arXiv}
Jasper De~Bock and Gert {d}e Cooman.
\newblock A desirability-based axiomatisation for coherent choice functions.
\newblock 2018.
\newblock ArXiv e-print: 1806.01044.

\bibitem{debock2019:interpretation}
Jasper De~Bock and Gert {d}e Cooman.
\newblock Interpreting, axiomatising and representing coherent choice functions
  in terms of desirability.
\newblock {\em Proceedings of Machine Learning Research}, 103:125--134, 2019.

\bibitem{debock2019:interpretation:arxiv}
Jasper De~Bock and Gert {d}e Cooman.
\newblock Interpreting, axiomatising and representing coherent choice functions
  in terms of desirability.
\newblock 2019.
\newblock ArXiv e-print: 1903.00336.

\bibitem{ipmu2020decooman}
Gert de~Cooman.
\newblock Coherent and {A}rchimedean choice in general {B}anach spaces.
\newblock In {\em Information Processing and Management of Uncertainty in
  Knowledge-Based Systems (Proceedings of IPMU 2020)}, pages 180--194. Springer
  International Publishing, 2020.

\bibitem{ipmu2020decooman:arxiv}
Gert de~Cooman.
\newblock Coherent and {A}rchimedean choice in general {B}anach spaces.
\newblock 2020.
\newblock ArXiv e-print: 2002.05461.

\bibitem{decooman2015:coherent:predictive:inference}
Gert de~Cooman, Jasper De~Bock, and M\'arcio~Alves Diniz.
\newblock Coherent predictive inference under exchangeability with imprecise
  probabilities.
\newblock {\em Journal of Artificial Intelligence Research}, 52:1--95, 2015.

\bibitem{cooman2004a}
Gert {d}e Cooman and Enrique Miranda.
\newblock Weak and strong laws of large numbers for coherent lower previsions.
\newblock {\em Journal of Statistical Planning and Inference},
  138(8):2409--2432, 2008.

\bibitem{cooman2011b}
Gert {de}~Cooman and Enrique Miranda.
\newblock Irrelevance and independence for sets of desirable gambles.
\newblock {\em Journal of Artificial Intelligence Research}, 45:601--640, 2012.

\bibitem{cooman:2012:indnatexdesirs}
Gert de~Cooman and Enrique Miranda.
\newblock {Irrelevant and independent natural extension for sets of desirable
  gambles.}
\newblock {\em Journal of Artificial Intelligence Research}, 45:601--640, 2012.

\bibitem{cooman2011a}
Gert {d}e Cooman, Enrique Miranda, and Marco Zaffalon.
\newblock Independent natural extension.
\newblock {\em Artificial Intelligence}, 175:1911--1950, 2011.

\bibitem{cooman2010}
Gert {d}e Cooman and Erik Quaeghebeur.
\newblock Exchangeability and sets of desirable gambles.
\newblock {\em International Journal of Approximate Reasoning}, 53(3):363--395,
  2012.
\newblock Special issue in honour of Henry E.~Kyburg, Jr.

\bibitem{levi1980a}
Isaac Levi.
\newblock {\em The Enterprise of Knowledge}.
\newblock MIT Press, London, 1980.

\bibitem{levi1999:isipta:imprecise:indeterminate}
Isaac Levi.
\newblock Imprecise and indeterminate probabilities.
\newblock In Gert {d}e Cooman, Fabio~G. Cozman, Serafín Moral, and Peter
  Walley, editors, {\em ISIPTA '99: Proceedings of the First International
  Symposium on Imprecise Probabilities and Their Applications}, pages 258--265,
  1999.

\bibitem{quaeghebeur2012:itip}
Erik Quaeghebeur.
\newblock Introduction to imprecise probabilities.
\newblock chapter Desirability. John Wiley \& Sons, 2014.

\bibitem{quaeghebeur2015:statement}
Erik Quaeghebeur, Gert de~Cooman, and Filip Hermans.
\newblock Accept {\&} reject statement-based uncertainty models.
\newblock {\em International Journal of Approximate Reasoning}, 57:69--102,
  2015.

\bibitem{seidenfeld1995}
Teddy Seidenfeld, Mark~J. Schervish, and Jay~B. Kadane.
\newblock A representation of partially ordered preferences.
\newblock {\em The Annals of Statistics}, 23:2168--2217, 1995.
\newblock Reprinted in \cite{seidenfeld1999}, pp.~69--129.

\bibitem{seidenfeld1999}
Teddy Seidenfeld, Mark~J. Schervish, and Jay~B. Kadane.
\newblock {\em Rethinking the Foundations of Statistics}.
\newblock Cambridge University Press, Cambridge, 1999.

\bibitem{seidenfeld2010}
Teddy Seidenfeld, Mark~J. Schervish, and Joseph~B. Kadane.
\newblock Coherent choice functions under uncertainty.
\newblock {\em Synthese}, 172(1):157--176, 2010.

\bibitem{troffaes2007}
Matthias C.~M. Troffaes.
\newblock Decision making under uncertainty using imprecise probabilities.
\newblock {\em International Journal of Approximate Reasoning}, 45(1):17--29,
  2007.

\bibitem{troffaes2013:lp}
Matthias C.~M. Troffaes and Gert {d}e Cooman.
\newblock {\em Lower Previsions}.
\newblock Wiley, 2014.

\bibitem{2017vancamp:phdthesis}
Arthur Van~Camp.
\newblock {\em Choice Functions as a Tool to Model Uncertainty}.
\newblock PhD thesis, Ghent University, Faculty of Engineering and
  Architecture, 2018.

\bibitem{vancamp2018:exchangeability}
Arthur Van~Camp and Gert {d}e Cooman.
\newblock Exchangeable choice functions.
\newblock {\em International Journal of Approximate Reasoning}, 100:85--104,
  2018.

\bibitem{2018vancamp:lexicographic}
Arthur Van~Camp, Gert {d}e Cooman, and Enrique Miranda.
\newblock Lexicographic choice functions.
\newblock {\em International Journal of Approximate Reasoning}, pages 97--119,
  2018.

\bibitem{vancamp2018:indifference}
Arthur Van~Camp, Gert {d}e Cooman, Enrique Miranda, and Erik Quaeghebeur.
\newblock Coherent choice functions, desirability and indifference.
\newblock {\em Fuzzy Sets and Systems}, 341:1--36, 2018.

\bibitem{walley1991}
Peter Walley.
\newblock {\em Statistical Reasoning with Imprecise Probabilities}.
\newblock Chapman and Hall, London, 1991.

\bibitem{walley2000}
Peter Walley.
\newblock Towards a unified theory of imprecise probability.
\newblock {\em International Journal of Approximate Reasoning}, 24:125--148,
  2000.

\end{thebibliography}

\end{document}